\documentclass[
]{ceurart}

\usepackage{listings}
\usepackage{array}
\usepackage[most]{tcolorbox}
\usepackage{epigraph}
\usepackage{amsmath}
\usepackage{amssymb}
\usepackage{amsthm}

\theoremstyle{plain}
\newtheorem{theorem}{Theorem}[section]

\newtheorem{proposition}[theorem]{Proposition}
\newtheorem{corollary}[theorem]{Corollary}

\theoremstyle{definition}
\newtheorem{definition}[theorem]{Definition}

\theoremstyle{remark}
\newtheorem{remark}[theorem]{Remark}

\usepackage{graphicx}
\usepackage{pgfplots}
\usepackage{xcolor}
\usepackage{tikz}
\usetikzlibrary{arrows.meta}
\pgfplotsset{compat=1.18}
\usepgfplotslibrary{fillbetween,groupplots}

\newif\ifdraft
\draftfalse 
\ifdraft
  \usepackage[colorinlistoftodos]{todonotes}
\else
  \usepackage[disable]{todonotes}
\fi

\begin{document}

\copyrightyear{2026}
\copyrightclause{Copyright for this paper by its authors.
  Use permitted under Creative Commons License Attribution 4.0
  International (CC BY 4.0).}

\conference{AAAI-26 Workshop on Machine Ethics: from formal methods to emergent machine ethics, January 27, 2026, Singapore}

\title{Bounded Morality}
\title[mode=alt]{Defining the Space of Moral Computation}

\author[1]{Max Kanwal}[%
orcid=0009-0008-7021-3227,
email=kanwal@stanford.edu,
url=https://linkedin.com/in/mkanwal/,
]
\fnmark[1]
\cormark[1]
\address[1]{Stanford University}

\author[2]{Caryn Tran}[%
orcid=0000-0002-4645-6607,
email=caryn@u.northwestern.edu,
url=https://linkedin.com/in/caryntran/,
]
\fnmark[1]
\address[2]{Northwestern University}

\author[3]{Patrick Mineault}[%
orcid=0000-0001-5519-842X,
email=patrick.mineault@gmail.com,
url=https://www.linkedin.com/in/pmineault/,
]
\address[3]{Amaranth Foundation}

\fntext[1]{These authors contributed equally.}
\cortext[1]{Corresponding author.}

\begin{abstract}
  Moral cognition has traditionally been modeled as adherence to fixed ethical theories—deontology, consequentialism, virtue ethics—implemented as static rules or value functions. We propose Bounded Morality, a formal framework for analyzing the computational demands of moral problems faced by finite agents. Extending Herbert Simon’s notion of bounded rationality, we formalize moral situations along two orthogonal dimensions: moral breadth, the scope of entities treated as morally relevant, and moral depth, the inferential integration required to evaluate their interactions. Limited resources impose an unavoidable tradeoff between these dimensions, defining a feasible space of moral computation. Within this space, ethical theories correspond to locally efficient strategies adapted to different demand regimes rather than competing accounts of moral truth. The framework yields a formal notion of moral regret and moral progress under constraint, and implies that moral alignment in artificial systems depends on the scaling and allocation of moral reasoning capacity rather than on direct imitation of human judgments.
\end{abstract}


\maketitle

\epigraph{\emph{“Human rational behavior is shaped by a scissors whose two blades are the structure of the task environment and the computational capabilities of the actor.”}}{--- Herbert A. Simon}

\section{Introduction}

\todo[inline]{check citations; update keywords}

Most attempts to formalize morality in artificial systems proceed by encoding one or more established ethical theories—utilitarian, deontological, contractualist, or virtue-based—into algorithmic decision rules \cite{takeshita_towards_2023, hegde_ethics_2020, white_mapping_2024, preniqi_moralbert_2024}. This strategy treats moral disagreement as a problem of theory selection: the central task is to identify which moral principles an agent ought to implement \cite{gabriel_artificial_2020}. While this approach specifies what actions are morally preferred for some situations, it is often not generalizable because it leaves largely unexamined how moral reasoning itself is carried out, how it scales with cognitive capacity, and why distinct moral theories recur across cultures and historical periods.\todo{cite; also why is this relevant?}

A complementary line of work spanning bounded rationality, rational analysis, and resource-rational cognition has shown that human judgment reflects adaptive trade-offs under computational constraint \cite{simon_bounded_1990, anderson_adaptive_2013, griffiths_rational_2015}. Moral cognition has increasingly been analyzed through this lens \cite{levine_resource-rational_2024}, including within AI alignment research \cite{levine_resource_2025}. Related work on \emph{bounded ethicality} examines how moral behavior is systematically shaped or distorted by cognitive and situational constraints \cite{chugh_bounded_2005, tenbrunsel_13_2008}. Together, these traditions describe the limits of moral agents and the strategies they use to cope with those limits. What remains underspecified is the demand side of moral reasoning. That is, which features of moral situations impose the informational and computational requirements that burden those resources in the first place.\todo{expand on RRC and bounded ethicality, maybe later?}

\textbf{Our contribution is to characterize the demand structure of moral problems and the trade-offs these demands force on bounded agents.} We ask how moral situations vary in ways that systematically increase or decrease the computational burden of moral reasoning. Drawing on developmental, situational, and comparative evidence, we show that moral dilemmas differ along at least two orthogonal dimensions: how many entities, groups, or timescales are treated as morally relevant, and how much inference, deliberation, and information integration is required to evaluate them.\todo{do we indeed do this?} These dimensions jointly define the demand structure of moral reasoning and therefore determine when and how trade-offs become unavoidable for finite agents.

\begin{tcolorbox}[
  colback=red!3,
  colframe=red!60!black,
  title=\textbf{Overview},
  fonttitle=\bfseries,
  boxrule=0.6pt,
  arc=2pt,
  left=6pt,
  right=6pt,
  top=6pt,
  bottom=6pt
]
In this paper, we propose a resource-aware theory—\emph{Bounded Morality}—that characterizes moral reasoning as constrained inference under situational demands. The first dimension, \emph{moral breadth}, captures the scope and resolution of moral representation: which entities, groups, timescales, or abstractions are treated as morally relevant. The second dimension, \emph{moral depth}, captures inferential complexity: how richly, recursively, and coherently those represented entities and their interactions are reasoned about. Together, breadth and depth determine the informational and inferential demands that moral situations impose on finite agents.
\end{tcolorbox}

This framework is consistent with Herbert Simon’s notion of bounded rationality \citep{simon_bounded_1990}, but relocates the focus from agent limitations to problem structure. Rather than introducing new bounds on cognition, we specify the structure of moral problems that generates unavoidable trade-offs. We model morality at the computational level (à la \citet{marr_vision_2010}) to distinguish problem demands from solution strategies.

Framing morality in this way reveals a structural trade-off. Expanding moral scope increases representational dimensionality, while deepening moral inference compounds computational and data requirements. Because resources are finite, only a subset of this breadth–depth space is feasible. Consequently, ethical theories can be understood as locally efficient strategies adapted to different resource demands rather than as competing prescriptions. 

For AI alignment, this reframing has a direct implication: alignment inherits the same demand structure as human moral reasoning. If human morality reflects bounded cognition, directly imitating human judgments risks reproducing those constraints. A more promising direction is to design systems whose moral reasoning expands with representational and computational capacity, extending the developmental logic by which moral competence scales with resources. This does not involve copying human judgments, but instead replicating the scaling relationship between resources and moral capacity. Alignment should therefore focus on how moral reasoning capacity is allocated and expanded, rather than on learning human morality itself.

This paper does not argue for any particular moral values or ethical theory. It studies a different problem: how moral reasoning must operate when agents face situational demands that exceed their available time, information, and computational capacity. When we speak of moral progress or optimality, we mean improved performance relative to a fixed moral objective under fixed constraints, not convergence on a uniquely correct morality. The goal is to clarify the trade-offs that any finite moral agent—human or artificial—must face, regardless of which moral values one ultimately endorses.

The framework yields several consequences. First, it offers a unifying descriptive account of moral development: across individuals and societies, moral growth involves coordinated expansion along both dimensions as cognitive capacity, social information, and cultural scaffolding increase, allowing agents to meet higher moral demands. This pattern is supported by developmental and comparative evidence showing that moral reasoning expands through increasingly flexible trade-offs between representational scope and inferential sophistication rather than wholesale replacement of earlier modes.\todo{cite}

Additionally, the framework clarifies the persistence of moral disagreement: even when agents share underlying values and face the same situation, they may allocate limited moral resources differently in response to situational demands, occupying different points on the feasible breadth–depth frontier. Finally, it provides a precise sense in which moral progress can be defined. Holding resource budgets fixed, progress consists in achieving better approximations to higher-capacity moral reasoning—lower regret, better allocations, or more effective abstractions—within the same feasible frontier. Longer-term progress occurs when cultural, institutional, or technological changes alter the demand structure itself, shifting the frontier outward.\todo{turn this into a implications box}

\begin{tcolorbox}[
  colback=red!3,
  colframe=red!60!black,
  title=\textbf{Implications},
  fonttitle=\bfseries,
  boxrule=0.6pt,
  arc=2pt,
  left=6pt,
  right=6pt,
  top=6pt,
  bottom=6pt
]
\textbf{1. Persistent Disagreement Without Value Conflict.}
Agents who share the same underlying values can still disagree if they allocate limited moral resources differently—distinguishing more entities (breadth) or tracing consequences further (depth). Disagreement may therefore reflect different approximations to the same ideal standard.

\textbf{2. Strategy-Relative Moral Evaluation.}
Under constraints, moral reasoning is an approximation problem. Strategies are best compared by how much regret they incur relative to an unbounded ideal, given the same resource limits.

\textbf{3. Two Forms of Moral Progress.}
Progress can occur by (i) using fixed resources more effectively—reducing regret within the same feasible frontier—or (ii) expanding capacity through cultural, institutional, or technological change, shifting the frontier outward.
\end{tcolorbox}

The remainder of the paper proceeds as follows. Section~\ref{sec:framework} derives the breadth–depth space of moral computation from the representational and inferential demands imposed by moral situations. Section~\ref{sec:bounded} formalizes Bounded Morality as constrained inference, defining resource costs, the resulting Pareto frontier, and strategy-relative moral regret. Section~\ref{sec:ethical_strategies} reinterprets canonical ethical theories as resource-bounded strategies occupying distinct regions of this space. We conclude by discussing implications for moral disagreement, moral progress, and AI alignment.\todo{double check this sign posting; not sure about use of the word ``derive''}

\section{Deriving the Space of Moral Computation}
\label{sec:framework}

We derive the structure of moral computation by treating moral reasoning as a form of bounded inference performed by cognitively limited agents embedded in complex social worlds. Rather than positing abstract dimensions a priori, we identify recurrent constraints and dissociations that appear across moral psychology, developmental theory, neuroscience, and computational modeling. Together, these literatures converge on two largely independent axes along which moral reasoning varies: the \emph{scope of moral representation} and the \emph{depth of moral inference}.

\subsection{Moral Breadth: Scope of Representation}

The first axis concerns \emph{what is morally represented at all}. We refer to this dimension as \emph{moral breadth}. Moral breadth captures the scope of entities, interests, and temporal horizons that an agent includes within its moral model.

A large empirical literature supports the existence of systematic variation along this dimension. Developmental studies show that young children begin with predominantly egocentric moral concern and gradually expand their scope to include peers, social groups, and abstract principles \citep{piaget_origins_1952, kohlberg_moral_1976}. Prosocial development research documents a shift from self-oriented reasoning toward concern for others’ welfare, fairness, and justice \citep{eisenberg_prosocial_2006}. In adulthood, however, moral scope remains highly variable: individuals differ markedly in how far they extend moral concern beyond immediate in-groups.

This variation has been operationalized directly. \citet{crimston_moral_2016} introduce the Moral Expansiveness Scale, which quantifies how broadly individuals extend moral standing—to outgroups, non-human animals, ecosystems, and future generations. These differences predict real-world moral judgments and policy attitudes, providing evidence that moral breadth is a measurable and psychologically meaningful dimension.

Philosophical and cultural traditions further reinforce this axis. Singer's \citep{singer_expanding_1981} “expanding circle” frames moral progress as widening the set of beings whose welfare is taken into account. Environmental ethics and intergenerational justice emphasize the extension of moral concern across space and time \citep{parfit_reasons_1986, gardiner_perfect_2011}. Importantly, this expansion imposes informational costs: broader moral representations require tracking more stakeholders, preferences, and interactions, increasing representational and coordination demands.

Operationally, moral breadth can be manipulated or measured through scope variation tasks (e.g., changing the number or type of affected parties), through temporal framing (short-term versus long-term consequences), or through explicit inclusion and exclusion of entities (humans, non-human animals, or ecosystems). Such manipulations reliably affect moral judgments even when task structure and reasoning demands are held constant \citep{parfit_reasons_1986, waytz_who_2010, crimston_moral_2016}.

\subsection{Moral Depth: Inferential Complexity}

The second axis concerns \emph{how moral information is processed}. We refer to this dimension as \emph{moral depth}. Moral depth captures the complexity of the inferential transformations applied to a given moral representation: how many steps of reasoning are performed, how many perspectives are integrated, and whether agents engage in counterfactual, recursive, or principle-based deliberation.

Evidence for this dimension comes from several converging lines of research. Dual-process theories distinguish fast, intuitive moral responses from slower, deliberative reasoning \citep{haidt_emotional_2001, greene_neural_2004}. Importantly, these processes can operate over the same moral inputs but yield different outcomes, demonstrating that inferential depth varies independently of representational scope. Cognitive load and time pressure selectively impair deeper forms of moral reasoning while leaving surface judgments intact \citep{greene_cognitive_2008, conway_deontological_2013}.

Developmental psychology provides further support. \citet{selman_growth_1980} shows that children progress from egocentric reasoning to the ability to coordinate multiple perspectives—a capacity that emerges later than basic inclusion of others. Piaget showed that young children place greater emphasis on outcome but come to consider intention more as they mature \cite{piaget_moral_2013}. Recent work shows that adults under cognitive load behave more similarly to children, shifting from intent-based evaluation to outcome-based \cite{buon_non-mentalistic_2013, martin_effect_2021}. Kohlberg’s postconventional stages emphasize not broader concern per se, but the ability to reason about principles, conflicts, and justifications abstracted across cases \cite{kohlberg_moral_1976}. Neurocognitive evidence aligns with this view: deeper moral reasoning recruits executive control networks associated with working memory, abstraction, and conflict resolution \citep{greene_beyond_2014}.

Crucially, moral depth is dissociable from moral breadth. An agent may include many stakeholders yet rely on shallow rules (e.g., equal division), or focus narrowly on a single individual while engaging in rich counterfactual and perspective-based reasoning. This dissociation has been observed empirically: individuals with broad moral concern do not necessarily show greater deliberative sophistication, and vice versa \citep{crimston_moral_2016, kahane_beyond_2018, fetherstonhaugh_insensitivity_1997}.

Operationally, moral depth can be studied using tasks that ask people to explain their decisions step by step, to think about what would happen under different possible actions or from different viewpoints, or to choose between rules that come into conflict. Slower response times \cite{suter_time_2011, paxton_reflection_2012}, greater disruption under time pressure or mental load \cite{greene_cognitive_2008}, and more complex explanations \citep{baron_protected_1997} all point to deeper moral reasoning, which tends to require more time and mental effort.

\subsection{The Breadth--Depth Tradeoff}

While analytically distinct, moral breadth and moral depth are jointly constrained by limited cognitive resources. Expanding the scope of representation increases informational load; increasing inferential depth raises computational demands. Empirically, individuals can trade one for the other: broad moral scope is frequently accompanied by simplified reasoning and concern \cite{fetherstonhaugh_insensitivity_1997, dickert_scope_2015}, while deep deliberation is typically restricted to narrower contexts \cite{suter_time_2011, greene_cognitive_2008, paxton_reflection_2012, brandt_handbook_2016}.

This tradeoff is mirrored in formal results from social choice and computational complexity. Eliciting and coordinating multiple perspectives and preferences quickly becomes intractable as the number of stakeholders grows \cite{arrow_social_1970, brandt_handbook_2016, conitzer_communication_2005, conitzer_social_2024}. As a result, both humans and institutions rely on abstraction, heuristics, and procedural shortcuts to manage moral complexity at scale \cite{simon_bounded_1990, gigerenzer2011heuristics, arrow_social_1970, sunstein_social_1996}.

The feasible configurations of moral computation therefore form a constrained region in a two-dimensional space defined by breadth and depth. Different moral strategies—heuristic rules, principled deliberation, institutional norms—correspond to different allocations within this space. In the following sections, we formalize this tradeoff, characterize the resulting Pareto structure, and use it to reinterpret moral disagreement, ethical theories, and moral progress within a unified computational framework.

\begin{table}[t]
\centering
\small
\begin{tabular}{@{}>{\raggedright\arraybackslash}p{2.75cm}
                >{\raggedright\arraybackslash}p{2.50cm}
                p{8.cm}@{}}
\toprule
Role in Framework & Concept & Interpretation \\
\midrule
Overall framework 
& Bounded Morality 
& Morality understood as a computational problem solved approximately under finite cognitive, computational, and data constraints. \\

Core axis (representation) 
& Moral Breadth 
& How much of the moral world is explicitly represented: which entities, stakeholders, timescales, or abstractions are taken into account. \\

Core axis (reasoning) 
& Moral Depth 
& How extensively consequences and interactions are propagated and integrated through inference, counterfactual reasoning, or deliberation. \\

Fundamental constraint 
& Resource Budget 
& The finite capacity available for moral reasoning, including attention, memory, computation, time, and data. \\

Agent-level choice 
& Moral Strategy 
& A policy for allocating limited resources across breadth and depth when evaluating moral decisions. \\

Idealized reference 
& Unbounded Moral Agent 
& A hypothetical oracle that represents all morally relevant entities and reasons exhaustively over their interactions. \\

Performance measure 
& Moral Regret 
& The expected loss in moral evaluation incurred by using a bounded approximation instead of the unbounded ideal. \\

Structural tradeoff 
& Pareto Frontier 
& The set of non-dominated breadth--depth allocations achievable under a fixed resource budget. \\

Normative implication 
& Moral Progress 
& Improving moral performance within fixed resource constraints, by achieving lower moral regret or better approximations to higher-capacity reasoning along the feasible breadth--depth frontier. \\
\bottomrule
\end{tabular}
\caption{\textbf{Conceptual summary of the Bounded Morality framework.}}
\label{tab:bounded_morality_conceptual}
\end{table}

\section{Bounded Morality as Constrained Computation}
\label{sec:bounded}

This section develops a formal model of moral reasoning as a constrained computational problem. 
The central idea is simple: moral evaluation consists of assessing the long-run consequences of interventions in a structured system. 
When computational resources are limited, an agent must decide (i) how finely to represent the system and (ii) how far forward to propagate consequences. 
These two choices—\textit{representational breadth} and \textit{inferential depth}—define the core trade-off of bounded morality.

\subsection{The Moral System}

\subsubsection*{Moral Interaction Structure}

\begin{definition}[Moral Interaction Graph]
A \emph{moral interaction graph} is a finite graph
\[
G^\star=(V^\star,E^\star),
\]
whose nodes $v\in V^\star$ represent morally relevant entities and whose edges encode direct influence relationships. 
An edge $(u,v)\in E^\star$ indicates that the morally relevant state of $v$ may depend directly on the state or treatment of $u$.
\end{definition}

Each node $v$ has a local state space $\mathcal{S}_v$. 
The joint moral state is
\[
s=(s_v)_{v\in V^\star}\in\mathcal{S}:=\prod_{v\in V^\star}\mathcal{S}_v.
\]

Let $A$ be a finite action set. 
For $H\in\mathbb{N}$, define the set of action sequences
\[
\mathcal{A}_H := A^{H}.
\]
An element $\alpha=(a_0,\dots,a_{H-1})\in\mathcal{A}_H$ represents a temporally extended intervention.

\subsubsection*{Intervention-Driven Dynamics}

\begin{definition}[Moral Dynamics]
Given world state $\omega\in\Omega$ and action sequence $\alpha\in\mathcal{A}_H$, the moral state evolves as
\[
s(0)=I(a_0,\omega), 
\qquad
s(t+1)=F(s(t),a_t),
\]
where $F:\mathcal{S}\times A\to\mathcal{S}$ satisfies the locality condition
\[
s_v(t+1)=F_v\Big(s_v(t),\{s_u(t):(u,v)\in E^\star\},a_t\Big).
\]
\end{definition}

Thus interventions induce trajectories on a graph-structured dynamical system. 
Moral reasoning consists of evaluating these trajectories.

\subsubsection*{Ground-Truth Moral Value}

\begin{definition}[Graph-Structured Welfare]
An instantaneous welfare function $U:\mathcal{S}\to\mathbb{R}$ is \emph{graph-structured} if
\[
U(s)
=
\sum_{v\in V^\star}\psi_v(s_v)
+
\sum_{(u,v)\in E^\star}\psi_{uv}(s_u,s_v),
\]
for some node potentials $\{\psi_v\}_{v\in V^\star}$ 
and edge potentials $\{\psi_{uv}\}_{(u,v)\in E^\star}$.
\end{definition}

Fix a discount factor $\gamma\in(0,1)$.

\begin{definition}[Ground-Truth Moral Value]
For action sequence $\alpha\in\mathcal{A}_\infty$,
\[
M^\star(\alpha,\omega)
:=
\sum_{t=0}^{\infty}\gamma^t\,U(s(t)).
\]
\end{definition}

The unboundedly optimal intervention is
\[
\alpha^\star(\omega)
\in
\arg\max_{\alpha} M^\star(\alpha,\omega).
\]

This defines the full-information moral control problem.

\subsection{Bounded Representations: Breadth and Depth}

In practice, agents cannot evaluate $M^\star$ exactly. 
They restrict both how much of the system they explicitly represent and how far forward they propagate consequences.

\subsubsection*{Abstraction and Breadth}

\begin{definition}[Abstract Moral Representation]
An \emph{abstract representation} consists of a graph
\[
G=(V,E)
\]
together with a surjective aggregation map
\[
\pi_V:V^\star\to V,
\]
which partitions morally relevant entities into aggregated units.

The vertices $V$ represent equivalence classes under $\pi_V$. 
Edges are induced by coarse interaction:
\[
(u,v)\in E
\iff
\exists u'\in\pi_V^{-1}(u),\ 
\exists v'\in\pi_V^{-1}(v)
\ \text{s.t.}\ 
(u',v')\in E^\star.
\]
\end{definition}

Thus $G$ is the coarse-grained projection of the true moral interaction graph $G^\star$, obtained by aggregating entities and inheriting interactions between their aggregates.

\begin{definition}[Admissible Action Sequences Under Abstraction]
Given representation $G$, admissible sequences are
\[
\mathcal{A}_H(G)
:=
\Big\{\
\alpha\in\mathcal{A}_H:
a_{t,v'}=\tilde a_{t,\pi_V(v')}
\ \forall v'\in V^\star,\ \forall t\
\Big\}.
\]
\end{definition}

Thus entities aggregated together must be treated identically at every time step.

\begin{definition}[Breadth]
The \emph{breadth} of representation $G$ is
\[
b(G):=|V|+|E|.
\]
\end{definition}

Breadth measures how many distinctions and interactions are explicitly modeled (i.e., the representational resolution).

\subsubsection*{Depth as Truncated Propagation}

\begin{definition}[Depth]
The \emph{depth} $H\in\mathbb{N}$ is a rollout horizon specifying how many steps of the dynamics are explicitly evaluated.
\end{definition}

\begin{definition}[Finite-Horizon Approximation]
Given $(G,H)$, define the truncated objective
\[
\hat{M}_{G,H}(\alpha,\omega)
=
\sum_{t=0}^{H}\gamma^t\,U(s(t)),
\]
where $\alpha\in\mathcal{A}_H(G)$ and $s(t)$ evolves under the true dynamics.
\end{definition}

The approximation error reflects only the structural constraints imposed by aggregation ($G$) and horizon truncation ($H$), not uncertainty about the underlying moral primitives or dynamics.

\subsection{Informational and Inferential Costs}

We associate computational cost with representational complexity and the effort required to reason over that representation.

\begin{definition}[Total Computational Cost]
\[
\mathrm{Cost}(G,H)
=
\mathrm{Cost}_{\mathrm{info}}(G)
+
\mathrm{Cost}_{\mathrm{infer}}(G,H).
\]
\end{definition}

\begin{definition}[Informational Cost]
\[
\mathrm{Cost}_{\mathrm{info}}(G)
=
f\big(b(G)\big),
\]
where $f$ is strictly increasing in representational breadth $b(G)=|V|+|E|$.
\end{definition}

Informational cost captures the resources required to encode the abstract interaction structure and its associated welfare and dynamical primitives at the chosen level of resolution.

\begin{definition}[Inferential Cost]
\[
\mathrm{Cost}_{\mathrm{infer}}(G,H)
=
\Theta\big(H\, b(G)\big)
\cdot
C_{\mathrm{search}}(G,H),
\]
where $C_{\mathrm{search}}(G,H)\ge 1$ denotes the effective number of candidate trajectories that must be evaluated to optimize action selection.
\end{definition}

The factor $\Theta(H\,b(G))$ reflects the cost of simulating a single length-$H$ trajectory under abstraction $G$, while $C_{\mathrm{search}}(G,H)$ captures the combinatorial complexity of optimizing over admissible action sequences.

\subsection{Moral Strategies Under Resource Constraints}

Fix a computational budget $B>0$. 
Let $\mathsf{P}$ be a probability distribution over world states $\Omega$.

Let $\mathcal{G}$ denote the set of admissible abstract representations 
(as defined above), 
and let $\mathcal{A}_H(G)$ denote the admissible action sequences 
of length $H$ under representation $G$.

\subsubsection*{Resource Allocation (Strategy Level)}

\begin{definition}[Allocation Rule]
An \emph{allocation rule} is a mapping
\[
\rho:\Omega\times\mathbb{R}_+ \to \mathcal{G}\times\mathbb{N}_0,
\]
such that for each $(\omega,B)$,
\[
\rho(\omega,B)=(G,H),
\qquad
\mathrm{Cost}(G,H)\le B.
\]
\end{definition}

The allocation rule determines which entities are distinguished 
(construction of $V$), 
which interactions are modeled (construction of $E$), 
and how far consequences are propagated ($H$).

\subsubsection*{Planning Within a Representation (Policy Level)}

Given $(G,H)$ and world state $\omega$, define the induced policy
\[
\pi_{G,H}(\omega)
\in
\arg\max_{\alpha\in\mathcal{A}_H(G)}
\hat{M}_{G,H}(\alpha,\omega).
\]

\subsubsection*{Bounded Moral Strategy}

\begin{definition}[Bounded Moral Strategy]
A \emph{bounded moral strategy} is specified by an allocation rule $\rho$. 
It induces a bounded moral policy
\[
\pi^\rho_B(\omega)
=
\pi_{G,H}(\omega),
\qquad
\text{where } (G,H)=\rho(\omega,B).
\]
\end{definition}

\begin{tcolorbox}[
  colback=red!3,
  colframe=red!60!black,
  title=\textbf{Example: Content Moderation as Bounded Morality},
  fonttitle=\bfseries,
  boxrule=0.6pt,
  arc=2pt,
  left=6pt,
  right=6pt,
  top=6pt,
  bottom=6pt
]
A minimal instantiation of the framework (full details in Appendix~\ref{app:worked_example}).

\textbf{Ground Truth.}
Let $G^\star=(V^\star,E^\star)$ with 
$V^\star=\{E_1,E_2,C,M_1,M_2\}$ 
(extremists $E_i$, connector $C$, moderates $M_j$). 
Polarization diffuses over $G^\star$ under sanction-modified linear dynamics with slowly decaying resentment.

Instantaneous welfare penalizes polarization magnitude and disagreement:
$U(s)
=
-\frac{1}{5}\sum_v s_v^2
-
\frac{\lambda}{6}\sum_{\{u,v\}\in E^\star}(s_u-s_v)^2,$
and ground-truth value is
$M^\star(a,\omega)=\sum_{t=0}^{\infty}\gamma^t U(s(t)).$

\textbf{Competing Interventions.}
Sanction extremists only:
$a^{(E)}=\mathbf{1}_{\{E_1,E_2\}},$
or sanction extremists and the connector:
$a^{(EC)}=\mathbf{1}_{\{E_1,E_2,C\}}.$

\textbf{Depth Limitation.}
A bounded planner evaluates
$\hat M_{G^\star,H}(a,\omega)=\sum_{t=0}^{H}\gamma^t U(s(t)).$

\[
\begin{array}{c|cc}
H & a^{(E)} & a^{(EC)}\\
\hline
2  & -0.297 & \mathbf{-0.088} \\
10 & \mathbf{-0.581} & -0.715 \\
\infty & \mathbf{-0.634} & -0.955
\end{array}
\]

At $H=2$, the planner selects $a^{(EC)}$.
At $H=10$, it selects $a^{(E)}$.
Under the infinite-horizon objective, $a^{(E)}$ is optimal, so shallow depth incurs regret $\approx 0.322$.

\textbf{Breadth Limitation.}
If $E_1,E_2,C$ are aggregated into a single abstract node, admissible actions must treat them identically, eliminating $a^{(E)}$. Even with large depth, the planner cannot implement the optimal intervention, inducing the same regret magnitude.

\textbf{Interpretation.}
Limited depth hides delayed backlash effects.
Limited breadth removes selective interventions.
Moral progress corresponds to reallocating computational resources across depth and breadth to reduce expected regret.
\end{tcolorbox}

\subsection{Regret and Moral Progress}

\begin{definition}[State-Dependent Regret]
For world state $\omega$, the regret of bounded moral strategy $\rho$ 
under budget $B$ is
\[
R(\omega;\rho,B)
=
M^\star(\alpha^\star(\omega),\omega)
-
M^\star(\pi^\rho_B(\omega),\omega).
\]
\end{definition}

\begin{definition}[Expected Regret]
\[
\mathbb{E}_{\omega \sim \mathsf{P}}
\big[
R(\omega;\rho,B)
\big].
\]
\end{definition}

A bounded moral strategy $\rho$ is distributionally efficient under $(\mathsf{P},B)$ if no other feasible strategy performs at least as well in nearly all likely world states (under $\mathsf{P}$) and strictly better in some of them.

\begin{definition}[Moral Progress]
Fix a distribution $\mathsf{P}$ over $\Omega$ and a budget $B>0$.
Let $\rho_1$ and $\rho_2$ be feasible bounded moral strategies 
(i.e., $\mathrm{Cost}(G,H)\le B$ for all $(G,H)=\rho_i(\omega,B)$).

We say that $\rho_2$ exhibits \emph{moral progress relative to} $\rho_1$ 
under $(\mathsf{P},B)$ if
\[
\mathbb{E}_{\omega\sim\mathsf{P}}
\big[
R(\omega;\rho_2,B)
\big]
<
\mathbb{E}_{\omega\sim\mathsf{P}}
\big[
R(\omega;\rho_1,B)
\big].
\]

In general, progress may arise from improvements in:
\begin{enumerate}
\item resource allocation (selection of the abstraction $(G,H)$),
\item estimation (learning the welfare potentials $\psi$ and dynamics $F$),
\item optimization (maximization of $\hat{M}_{G,H}$ over $\mathcal{A}_H(G)$).
\end{enumerate}

In this paper, we isolate the first source and focus on progress driven by improved allocation rules $\rho$. We assume that the moral primitives and dynamics are known, and that optimization within a fixed representation $(G,H)$ is exact, so that regret arises solely from structural resource allocation.
\end{definition}


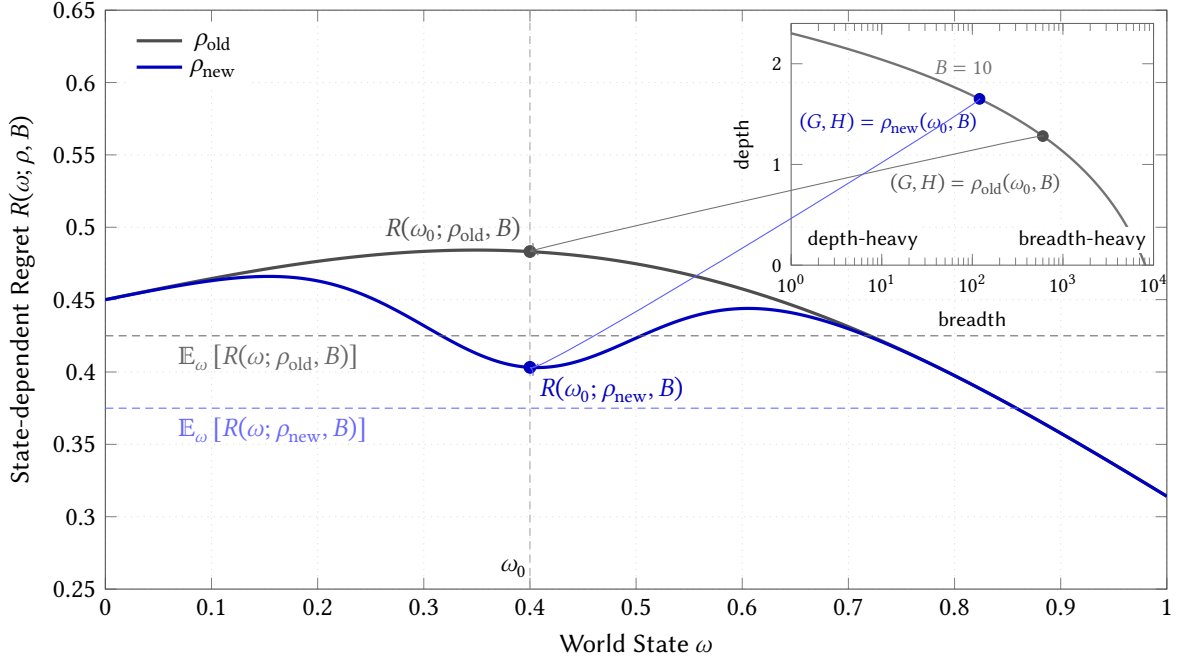
\begin{figure}[t]
\centering
\begin{tikzpicture}

\def\Bfixed{10}          
\def\omegaZero{0.40}     

\def\Eold{0.425}
\def\Enew{0.375}

\def\bMin{1}
\def\bMax{10000}

\def\bOld{600}
\def\bNew{120}

\newcommand{\dB}[1]{ln(\Bfixed - ln(#1))}

\begin{axis}[
    name=main,
    width=0.98\linewidth,
    height=0.58\linewidth,
    xmin=0, xmax=1,
    ymin=0.25, ymax=0.65,
    xlabel={World State $\omega$},
    ylabel={State-dependent Regret $R(\omega;\rho,B)$},
    title={Moral Progress via Improved Resource Allocation},
    tick label style={font=\small},
    label style={font=\small},
    title style={font=\small},
    grid=major,
    grid style={dotted, gray!25},
    axis line style={black!70},
    legend style={draw=none, font=\small, at={(0.02,0.98)}, anchor=north west},
    clip=true,
]

\addplot[
    very thick,
    black!70,
    domain=0:1,
    samples=500
]
{0.45 - 0.15*x + 0.10*sin(deg(3*x))};

\addplot[
    very thick,
    blue!75!black,
    domain=0:1,
    samples=500
]
{0.45 - 0.15*x + 0.10*sin(deg(3*x))
 - 0.08*exp(-((x-0.40)/0.15)^2)};

\addlegendentry{$\rho_{\mathrm{old}}$}
\addlegendentry{$\rho_{\mathrm{new}}$}

\addplot[densely dashed, black!60] coordinates {(0,\Eold) (1,\Eold)};
\addplot[densely dashed, blue!60]  coordinates {(0,\Enew) (1,\Enew)};

\node[font=\small, anchor=west, text=black!60]
  at (axis cs:0.06,\Eold-0.015)
  {$\mathbb{E}_{\omega}\left[R(\omega;\rho_{\mathrm{old}},B)\right]$};

\node[font=\small, anchor=west, text=blue!60]
  at (axis cs:0.06,\Enew-0.015)
  {$\mathbb{E}_{\omega}\left[R(\omega;\rho_{\mathrm{new}},B)\right]$};

\pgfmathsetmacro{\RoldAtOmega}{0.45 - 0.15*\omegaZero + 0.10*sin(deg(3*\omegaZero))}
\pgfmathsetmacro{\RnewAtOmega}{\RoldAtOmega - 0.08*exp(-((\omegaZero-0.40)/0.15)^2)}

\addplot[densely dashed, black!35] coordinates {(\omegaZero,0.25) (\omegaZero,0.65)};
\node[font=\small, anchor=south] 
  at (axis cs:\omegaZero-0.015,0.252) {$\omega_0$};

\addplot[only marks, mark=*, mark size=2.2pt, black!70]
  coordinates {(\omegaZero,\RoldAtOmega)};
\addplot[only marks, mark=*, mark size=2.2pt, blue!75!black]
  coordinates {(\omegaZero,\RnewAtOmega)};

\node[font=\small, anchor=south east, text=black!70]
  at (axis cs:\omegaZero-0.0,\RoldAtOmega+0.0)
  {$R(\omega_0;\rho_{\mathrm{old}},B)$};

\node[font=\small, anchor=north west, text=blue!75!black]
  at (axis cs:\omegaZero+0.0,\RnewAtOmega-0.0)
  {$R(\omega_0;\rho_{\mathrm{new}},B)$};

\coordinate (mainOldOmega) at (axis cs:\omegaZero,\RoldAtOmega);
\coordinate (mainNewOmega) at (axis cs:\omegaZero,\RnewAtOmega);

\end{axis}

\begin{axis}[
    name=inset,
    at={(main.north east)},
    anchor=north east,
    xshift=-5pt,
    yshift=-5pt,
    width=0.40\linewidth,
    height=0.30\linewidth,
    xmode=log,
    xmin=\bMin, xmax=\bMax,
    ymin=0, ymax=2.4,
    tick label style={font=\scriptsize},
    label style={font=\scriptsize},
    grid=none,
    xlabel={breadth},
    ylabel={depth},
    axis line style={black!60},
    clip=true,
]

\addplot[
    draw=black!55,
    line width=0.9pt,
    domain=\bMin:\bMax,
    samples=300
]
{\dB{x}};

\node[font=\scriptsize, anchor=west, text=black!55]
  at (axis cs:30,{ \dB{30} + 0.08 })
  {$B=\Bfixed$};

\node[font=\scriptsize, anchor=west, fill=white, inner sep=1pt]
  at (axis cs:1.4,0.25) {depth-heavy};
\node[font=\scriptsize, anchor=east, fill=white, inner sep=1pt]
  at (axis cs:9000,0.25) {breadth-heavy};

\pgfmathsetmacro{\dOld}{\dB{\bOld}}
\pgfmathsetmacro{\dNew}{\dB{\bNew}}

\addplot[only marks, mark=*, mark size=2.0pt, black!70]
  coordinates {(\bOld,\dOld)};
\addplot[only marks, mark=*, mark size=2.0pt, blue!75!black]
  coordinates {(\bNew,\dNew)};

\node[font=\scriptsize, anchor=east, align=right, text=black!70,
  xshift=10pt,     
  yshift=-18pt      
]
  at (axis cs:\bOld,\dOld)
  {$(G,H)=\rho_{\mathrm{old}}(\omega_0,B)$};

\node[font=\scriptsize, anchor=west, align=left, text=blue!75!black, 
  xshift=-71pt,
  yshift=-9pt
]
  at (axis cs:\bNew,\dNew)
  {$(G,H)=\rho_{\mathrm{new}}(\omega_0,B)$};

\coordinate (insetOldOmega) at (axis cs:\bOld,\dOld);
\coordinate (insetNewOmega) at (axis cs:\bNew,\dNew);

\end{axis}

\draw[->, thin, black!55]
  (insetOldOmega) .. controls +(0.15,0.08) and +(0.18,0.08) .. (mainOldOmega);

\draw[->, thin, blue!65]
  (insetNewOmega) .. controls +(0.15,0.02) and +(0.18,-0.08) .. (mainNewOmega);

\end{tikzpicture}

\caption{
At fixed budget $B$ and distribution $\mathsf{P}$, policies $\rho_{\mathrm{old}}$ and $\rho_{\mathrm{new}}$ map world states to structural allocations $(G,H)$. The plot highlights a representative state $\omega_0$ where the improved policy selects a different breadth--depth trade-off $(G,H)=\rho_{\mathrm{new}}(\omega_0,B)$ on the fixed-budget frontier, yielding lower state-dependent regret $R(\omega_0;\rho_{\mathrm{new}},B)$ than $R(\omega_0;\rho_{\mathrm{old}},B)$. Dashed lines indicate expected regret $\mathbb{E}_{\omega}[R(\omega;\rho,B)]$ (here shown for a uniform $\omega$).
}
\label{fig:moral_progress_allocation_mapping_omega0}
\end{figure}

\begin{remark}[Relation to Causality, Planning, and Control]
The formulation above places moral decision-making within the theory of controlled dynamical systems on graphs. Actions function as interventions that set initial conditions and influence local update rules, paralleling the use of do-operators in structural causal models. Moral evaluation then corresponds to planning in a dynamical system where value is assigned to entire trajectories induced by interventions.

In this interpretation, depth $H$ plays the role of a bounded planning horizon, while abstraction $(G,\pi_V)$ corresponds to state aggregation and parameter tying in approximate dynamic programming. The bounded morality problem can thus be viewed as a constrained optimal control problem in which representational resolution and rollout depth jointly restrict attainable performance.

Unlike standard control settings, however, the objective function is morally structured—it decomposes over individuals and their relationships. As a result, representational choices are ethically consequential. Coarse aggregation or truncated propagation can induce moral error even when the underlying welfare function is correct.
\end{remark}

\subsection{Asymptotic Scaling and a Canonical Cost Model} 
\label{sec:canonical_cost}

We now conjecture asymptotic bounds on informational and inferential costs, 
introduce a tractable canonical model consistent with these bounds, 
and derive the resulting breadth--depth trade-off under a fixed computational budget.

\subsubsection*{Asymptotic Bounds}

Let $b=b(G)=|V|+|E|$.

\begin{proposition}[Informational Lower Bound]
There exists $c_1>0$ such that
\[
\mathrm{Cost}_{\mathrm{info}}(G)
\ge
c_1\, b.
\]
\end{proposition}

\begin{proof}
Any abstraction distinguishing $|V|$ aggregates and $|E|$ interactions 
must encode at least these structural elements.
\end{proof}

\begin{proposition}[Inferential Lower Bound]
There exists $c_2>0$ such that
\[
\mathrm{Cost}_{\mathrm{infer}}(G,H)
\ge
c_2\, H b.
\]
\end{proposition}

\begin{proof}
Simulating a single trajectory of length $H$ requires 
$\Theta(b)$ work per step under locality, 
yielding $\Theta(Hb)$ total effort.
\end{proof}

Evaluating a single trajectory scales linearly in breadth and depth. Planning, however, may require comparing exponentially many trajectories in the worst case. When interactions are sufficiently localized, this growth can instead remain polynomial in the horizon.
Thus
\[
\Omega(Hb)
\;\le\;
\mathrm{Cost}_{\mathrm{infer}}(G,H)
\;\le\;
O\big(Hb\,C_{\mathrm{search}}(G,H)\big),
\]
where $C_{\mathrm{search}}(G,H)\ge 1$ captures search complexity.

\subsubsection*{Canonical Cost Model}

To isolate the structural trade-off, 
we adopt a tractable model that respects the lower bounds and abstracts from worst-case search effects.

\begin{definition}[Canonical Cost Model]
For $b=b(G)$,
\[
\mathrm{Cost}(b,H)
=
\alpha b
+
\beta b H^{p},
\]
where $\alpha,\beta>0$ and $p\ge 1$.
\end{definition}

The first term models representational cost; 
the second models rollout effort growing linearly in breadth and polynomially in depth.

\subsubsection*{Budget Constraint and Feasible Region}

Fix budget $B>0$. 
Admissible pairs satisfy
\[
\alpha b + \beta b H^{p} \le B.
\]
For $b>0$,
\[
H^{p}
\le
\frac{B}{\beta b}
-
\frac{\alpha}{\beta}.
\]

Whenever $B>\alpha b$, this yields
\[
H
\le
\left(
\frac{B}{\beta b}
-
\frac{\alpha}{\beta}
\right)^{1/p}.
\]

For $B \gg \alpha b$,
\[
H_{\max}(b)
=
\Theta(b^{-1/p}).
\]

\subsubsection*{Breadth--Depth Trade-Off}

\begin{corollary}[Inverse Scaling Law]
Under the canonical cost model with fixed $B$ and $p\ge 1$, 
there exists $C>0$ such that
\[
H_{\max}(b)
\le
C\, b^{-1/p}.
\]
\end{corollary}

\begin{proof}
From
\(
\alpha b + \beta b H^{p} \le B
\),
we obtain
\(
H^{p} \le \frac{B}{\beta b}-\frac{\alpha}{\beta}.
\)
For $b$ sufficiently small relative to $B/\alpha$, 
the right-hand side is bounded above by $C' b^{-1}$ 
for some $C'>0$. 
Taking $p$-th roots gives
\(
H \le C b^{-1/p}.
\)
\end{proof}

This corollary formalizes the structural tension at the heart of bounded morality: 
increasing representational resolution strictly reduces the horizon over which consequences can be propagated under fixed computational resources. 
The feasible region in $(b,H)$ space is therefore downward-sloping, and any allocation of finite resources must lie on or below this breadth--depth frontier.

\section{Ethical Theories as Resource-Bounded Moral Strategies}
\label{sec:ethical_strategies}

\epigraph{\emph{“We do not grow absolutely, chronologically. We grow sometimes in one dimension, and not in another; unevenly. We are relative. We are mature in one realm, childish in another.”}}{--- Anaïs Nin}

Section~\ref{sec:bounded} formalized moral reasoning as a constrained computational problem: under a budget $B$, an allocation rule $\rho$ selects an abstract representation $(G,\pi_V)$ and a rollout depth $H$, after which the agent optimizes a truncated objective over admissible intervention sequences $\alpha \in \mathcal{A}_H(G)$. Regret arises when structural restrictions on breadth and depth prevent recovery of the ground-truth optimum.

In this section, we reinterpret several canonical ethical theories as \emph{families of bounded moral strategies}. Each theory is modeled as a structured restriction on:
\begin{enumerate}
\item admissible allocation rules $\rho$ (hence on feasible $(G,H)$ pairs),
\item the aggregation functional used to evaluate trajectories,
\item and, where relevant, additional admissibility constraints on interventions.
\end{enumerate}
The underlying moral primitives $(G^\star,F,U)$ remain fixed. Differences between theories are represented as differences in how computational resources are allocated and how trajectory evaluations are constructed under constraint.

The mappings below are deliberately schematic. They isolate dominant structural commitments—breadth, depth, aggregation, and admissibility—rather than attempting to reconstruct full philosophical doctrines.

\subsection{Strategy Families}

Let $(G,H)=\rho(\omega,B)$ and let $s(0{:}H)$ denote the trajectory induced by $\alpha\in\mathcal{A}_H(G)$ under the true dynamics. A strategy family is characterized by an evaluation functional
\[
\mathcal{T}\Big(\{s(t)\}_{t=0}^{H};G\Big),
\]
and (optionally) a restricted admissible class $\mathcal{A}^{\mathrm{adm}}_H(G)\subseteq\mathcal{A}_H(G)$. The induced policy is
\[
\pi^{\rho,\mathcal{T}}_B(\omega)
\in
\arg\max_{\alpha\in\mathcal{A}^{\mathrm{adm}}_H(G)}
\mathcal{T}\Big(\{s(t)\}_{t=0}^{H};G\Big).
\]

Under the baseline model of Section~\ref{sec:bounded}, $\mathcal{T}$ coincides with the truncated discounted sum $\hat{M}_{G,H}$ and $\mathcal{A}^{\mathrm{adm}}_H(G)=\mathcal{A}_H(G)$. Ethical theories can be represented as structured deviations from this baseline.

\subsection{Act Utilitarianism}

Act utilitarianism evaluates interventions by maximizing total welfare across all affected individuals \citep{mill2016utilitarianism}. In the present formalism, its defining commitment is additive aggregation over a broad representational scope.

\paragraph{Aggregation.}
The evaluation functional coincides with the truncated ground-truth objective:
\[
\mathcal{T}_{\mathrm{UTIL}}\Big(\{s(t)\}_{t=0}^{H};G\Big)
=
\sum_{t=0}^{H}\gamma^t U(s(t)).
\]

\paragraph{Allocation structure.}
A utilitarian strategy family consists of allocation rules $\rho$ that prioritize representational breadth:
\[
\mathcal{R}_{\mathrm{UTIL}}
=
\Big\{
\rho:\ \rho(\omega,B)=(G,H)
\ \text{with}\ 
b(G)\ \text{maximized under }B,
\ H\ \text{moderate}
\Big\}.
\]
Depth is limited by budget but not minimized; breadth is expanded to include as many morally relevant entities and interactions as feasible.

\paragraph{Computational interpretation.}
Computation is devoted primarily to representing many stakeholders. Aggregation remains structurally simple. This strategy is locally efficient when interaction effects are weak or approximately decomposable, but becomes expensive when deep dependencies dominate.

\begin{table}[t]
\centering
\setlength{\tabcolsep}{3pt}
\begin{tabular}{@{}lcccc@{}}
\toprule
Theory
& Breadth
& Depth
& Structural Commitment
& Typical Regime \\
\midrule
Act Utilitarianism
& High
& Mod.
& Additive sum $\sum_{t}\gamma^t U(s(t))$
& Large, weakly coupled systems \\

Deontology
& Mod.
& Mod.
& Hard constraints on $\mathcal{A}_H(G)$
& Tight rules; limited scope \\

Contractualism
& Low
& High
& Worst-case $\min_v \sum_{t}\gamma^t U_v(s(t))$
& Few parties; deep objections \\

Virtue Ethics
& Very Low
& $0$
& Amortized policy $\pi_\theta$
& Real-time action \\

Care Ethics
& Local
& Mod.
& Relational (edge-weighted) terms
& Dense local dependence \\

\bottomrule
\end{tabular}

\caption{
\textbf{Ethical theories as caricatured families of bounded moral strategies.}
Each theory occupies a characteristic region of the breadth--depth frontier and imposes a distinct structural commitment on evaluation or admissibility.
}
\label{tab:ethical_strategies}
\end{table}

\subsection{Deontological Rule-Based Ethics}

Deontological theories evaluate interventions by compliance with rules or duties rather than by outcome maximization \citep{kant2020groundwork}. In the present framework, this corresponds primarily to restricting admissible interventions.

\paragraph{Admissibility and search reduction.}
Let $\mathcal{C}$ denote a set of rule constraints. Define
\[
\mathcal{A}^{\mathrm{adm}}_{H,\mathcal{C}}(G)
=
\{\alpha\in\mathcal{A}_H(G): \alpha \text{ satisfies } \mathcal{C}\}.
\]
The evaluation functional imposes a hard barrier:
\[
\mathcal{T}_{\mathrm{RULE}}\Big(\{s(t)\};G\Big)
=
\begin{cases}
\sum_{t=0}^{H}\gamma^t U(s(t)), & \alpha\in\mathcal{A}^{\mathrm{adm}}_{H,\mathcal{C}}(G),\\
-\infty, & \text{otherwise.}
\end{cases}
\]

The constraints reduce the effective search space from $\mathcal{A}_H(G)$ to the feasible subset $\mathcal{A}^{\mathrm{adm}}_{H,\mathcal{C}}(G)$, potentially yielding large savings in the search factor $C_{\mathrm{search}}(G,H)$. 

However, feasibility typically decreases as breadth and depth increase. Larger $b(G)$ introduces more entities and constraint instances; larger $H$ requires constraints to hold over longer trajectories. In general,
\[
\mathcal{A}^{\mathrm{adm}}_{H,\mathcal{C}}(G)
=
\bigcap_{t=0}^{H}
\{\alpha:\text{constraints hold at time }t\},
\]
so the admissible set shrinks with either $b(G)$ or $H$. In extreme cases it may be empty, corresponding to moral infeasibility under the chosen scope and horizon.

\paragraph{Allocation structure.}
A canonical deontological family therefore restricts both breadth and depth:
\[
\mathcal{R}_{\mathrm{RULE}}
=
\Big\{
\rho:\ \rho(\omega,B)=(G,H)
\ \text{with}\ 
b(G) \text{ and } H\ \text{medium-to-small}
\Big\}.
\]

\paragraph{Computational interpretation.}
Deontological strategies gain efficiency by pruning the search space rather than simplifying aggregation. Their tractability depends on maintaining medium breadth and depth; as scope or horizon expands, constraint-checking costs grow and the feasible set can collapse.

\subsection{Contractualism}

Contractualist theories evaluate interventions by whether they can be justified to each affected individual, often modeled via the strongest reasonable objection \citep{scanlon2000we}. The defining structural feature is worst-case aggregation over represented individuals.

\paragraph{Aggregation.}
Let $U_v(s)$ denote the portion of instantaneous welfare attributed to represented entity $v\in V$ under abstraction $G$. Define
\[
\mathcal{T}_{\mathrm{CONTRACT}}\Big(\{s(t)\}_{t=0}^{H};G\Big)
=
\min_{v\in V}
\sum_{t=0}^{H}\gamma^t U_v(s(t)).
\]

\paragraph{Allocation structure.}
Contractualist strategy families restrict breadth while increasing depth:
\[
\mathcal{R}_{\mathrm{CONTRACT}}
=
\Big\{
\rho:\ \rho(\omega,B)=(G,H)
\ \text{with}\ 
b(G)\ \text{small},\ 
H\ \text{large under }B
\Big\}.
\]

\paragraph{Computational interpretation.}
Attention is concentrated on a limited set of stakeholders, but substantial depth is allocated to propagate indirect consequences and objections. This strategy is viable in small-scale, high-stakes settings with sufficient deliberative resources.

\subsection{Virtue Ethics}

Virtue ethics emphasizes stable dispositions rather than explicit outcome calculation \citep{gottlieb2015aristotle}. In the present model, this corresponds to amortized inference.

\paragraph{Policy form.}
Set $H=0$ and define a learned policy
\[
\pi_\theta:\Omega\to A,
\]
where parameters $\theta$ encode dispositions acquired through experience.

\paragraph{Allocation structure.}
The associated strategy family satisfies
\[
\mathcal{R}_{\mathrm{VIRTUE}}
=
\Big\{
\rho:\ \rho(\omega,B)=(G,0)
\ \text{with}\ 
b(G)\ \text{minimal}
\Big\},
\qquad
\pi^\rho_B(\omega)=\pi_\theta(\omega).
\]

\paragraph{Computational interpretation.}
Online rollout is eliminated. Computation is shifted offline into learning. This strategy is efficient under extreme time constraints but sensitive to context.

\subsection{Care Ethics}

Care ethics emphasizes concrete relationships and contextual interdependence \citep{gilligan_different_1993}. Its defining feature in the present framework is localized representation of dense relational structure.

\paragraph{Representation.}
Let $v_0$ be a focal agent and let $\mathcal{N}_r(v_0)$ denote a radius-$r$ neighborhood in $G^\star$. Define
\[
G \approx G^\star[\mathcal{N}_r(v_0)].
\]

\paragraph{Aggregation.}
Evaluation emphasizes relational (edge) terms:
\[
\mathcal{T}_{\mathrm{CARE}}\Big(\{s(t)\}_{t=0}^{H};G\Big)
=
\sum_{t=0}^{H}\gamma^t
\sum_{(u,v)\in E}
\psi_{uv}(s_u(t),s_v(t)).
\]

\paragraph{Allocation structure.}
\[
\mathcal{R}_{\mathrm{CARE}}
=
\Big\{
\rho:\ \rho(\omega,B)=(G,H)
\ \text{with}\ 
G\ \text{local around }v_0,\ 
H\ \text{moderate}
\Big\}.
\]

\paragraph{Computational interpretation.}
Representational capacity is devoted to dense local interactions rather than global enumeration. Depth is sufficient to integrate relational dependencies within the local subgraph but does not scale globally.

\subsection{Ethical Disagreement as Allocation Disagreement}

On this formalization, ethical disagreement can arise from both disagreement about welfare primitives but also from disagreement about resource allocation. Distinct theories correspond to different structured restrictions on $(G,H)$, different aggregation functionals $\mathcal{T}$, and different admissibility constraints. Under fixed budget $B$ and distribution $\mathsf{P}$, these structural commitments induce different regret profiles.

Ethical theories can thus be interpreted as locally efficient regions of the breadth--depth frontier (Figure \ref{fig:ethical_strategies}). Each represents a characteristic solution to the problem of allocating limited computational resources across representational scope and inferential depth in morally structured dynamical systems.

\begin{figure}[t]
\centering
\begin{tikzpicture}
\begin{axis}[
    width=0.98\linewidth,
    height=0.58\linewidth,
    xmode=log,
    xmin=1, xmax=5.5e4,
    ymin=0, ymax=2.5,
    xlabel={Breadth $b$},
    ylabel={Depth $H$},
    title={Ethical Theories as Moral Strategy Families in Breadth--Depth Space},
    tick label style={font=\small},
    label style={font=\small},
    title style={font=\small},
    grid=major,
    grid style={dotted, line width=0.35pt, gray!25},
    axis line style={black!70},
    clip=true,
    axis on top,
]

\definecolor{contour}{RGB}{95,115,150}   
\definecolor{util}{RGB}{185,120,65}      
\definecolor{con}{RGB}{120,95,160}       
\definecolor{care}{RGB}{55,125,120}      
\definecolor{deon}{RGB}{145,65,65}       
\definecolor{virt}{RGB}{135,65,105}      

\def\BOne{10}
\def\BTwo{12}
\def\BThr{14}

\newcommand{\dBTwo}[1]{ln(\BTwo - ln(#1))}

\def\conL{1}
\def\conR{10}
\addplot[draw=none, fill=con, fill opacity=0.14]
  coordinates {(1,0) (\conL,{\dBTwo{\conL}})}
  -- plot[domain=\conL:\conR, samples=120] (\x,{ln(\BTwo - ln(\x))})
  -- (\conR,{\dBTwo{\conR}})
  -- (1,0) -- cycle;

\def\careL{15}
\def\careR{150}
\addplot[draw=none, fill=care, fill opacity=0.14]
  coordinates {(1,0) (\careL,{\dBTwo{\careL}})}
  -- plot[domain=\careL:\careR, samples=140] (\x,{ln(\BTwo - ln(\x))})
  -- (\careR,{\dBTwo{\careR}})
  -- (1,0) -- cycle;

\def\deonL{200}
\def\deonR{1000}

\addplot[draw=none, fill=deon, fill opacity=0.14]
  coordinates {(1,0) (\deonL,{\dBTwo{\deonL}})}
  -- plot[domain=\deonL:\deonR, samples=140] (\x,{ln(\BTwo - ln(\x))})
  -- (\deonR,{\dBTwo{\deonR}})
  -- (1,0) -- cycle;

\def\xMax{55000}
\def\utilL{2000}

\path[fill=util, fill opacity=0.14, draw=none]
  (axis cs:1,0)
  -- (axis cs:\utilL,{ln(\BTwo - ln(\utilL))})

  -- (axis cs:1200,{ln(\BTwo - ln(1200))})
  -- (axis cs:2000,{ln(\BTwo - ln(2000))})
  -- (axis cs:4000,{ln(\BTwo - ln(4000))})
  -- (axis cs:8000,{ln(\BTwo - ln(8000))})
  -- (axis cs:16000,{ln(\BTwo - ln(16000))})
  -- (axis cs:30000,{ln(\BTwo - ln(30000))})
  -- (axis cs:\xMax,{ln(\BTwo - ln(\xMax))})

  -- (axis cs:\xMax,0)   
  -- cycle;

\addplot[
  draw=contour!40,
  dashed,
  line width=0.6pt,
  domain=1:5.5e4,
  samples=260
] {ln(\BOne - ln(x))};

\addplot[
  draw=contour!70,
  line width=1.2pt,
  domain=1:5.5e4,
  samples=320
] {ln(\BTwo - ln(x))};

\addplot[
  draw=contour!50,
  dashed,
  line width=0.8pt,
  domain=1:5.5e4,
  samples=260
] {ln(\BThr - ln(x))};

\node[font=\small, text=contour!45, rotate=-39, anchor=west]
  at (axis cs:900, {ln(\BOne - ln(675))}) {$B=5$};
\node[font=\small, text=contour!70, rotate=-27, anchor=west]
  at (axis cs:1200, {ln(\BTwo - ln(800))}) {$B=10$};
\node[font=\small, text=contour!50, rotate=-20, anchor=west]
  at (axis cs:1300, {ln(\BThr - ln(800))}) {$B=15$};

\tikzset{
    theorylabel/.style={
    font=\small,
    inner sep=1.5pt,
    fill=white,
    fill opacity=0.85,
    text opacity=1,
    rounded corners=1pt
  }
}

\node[theorylabel, text=con!95!black, anchor=west]
  at (axis cs:1.1,2.) {Contractualism};

\node[theorylabel, text=care!95!black, anchor=west]
  at (axis cs:10,1.7) {Care Ethics};

\node[theorylabel, text=util!95!black, anchor=west]
  at (axis cs:300,0.60) {Utilitarianism};

\node[theorylabel, text=deon!95!black, anchor=west]
  at (axis cs:65,1.4) {Deontology};

\node[theorylabel, text=virt!95!black, anchor=west]
  at (axis cs:2.,0.10) {Virtue Ethics};


\addplot[only marks, mark=*, mark size=2.6pt, draw=virt!95!black, fill=virt]
  coordinates {(1.5,0.08)};
\addplot[only marks, mark=o, mark size=5.0pt, draw=virt, very thick, opacity=0.18]
  coordinates {(1.5,0.08)};

\end{axis}
\end{tikzpicture}
\caption{
Budget contours ($B$) induce Pareto-efficient breadth--depth frontiers under a canonical cost model.
Ethical theories are depicted as strategy families with characteristic scaling tendencies: utilitarianism trades budget for breadth, contractualism for depth, care ethics for local breadth and moderate depth, while deontology and virtue ethics remain near low-depth regimes.
Regions are conceptual caricatures.
}
\label{fig:ethical_strategies}
\end{figure}
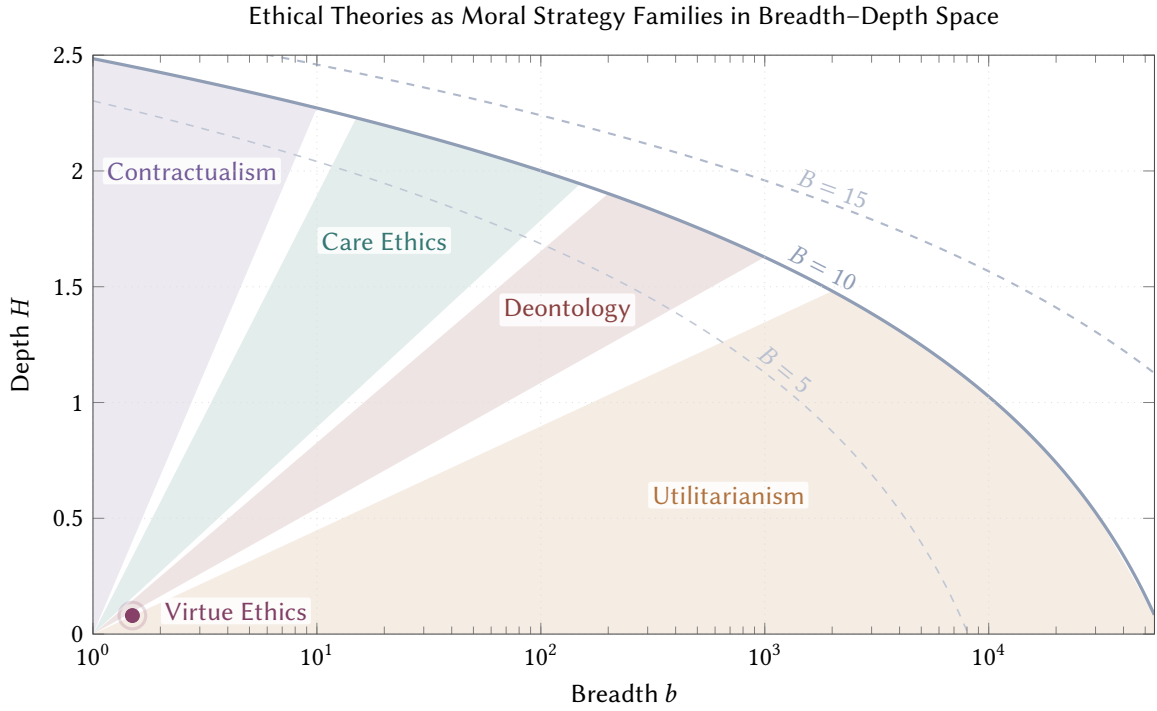

\section{Discussion}

This paper advances a limited but we believe, clarifying perspective on moral reasoning: even if moral truth is well defined, reasoning about it is a constrained computational problem. By making representational scope, inferential depth, and computational costs explicit, the framework highlights structural tradeoffs that any finite moral agent—human or artificial—must confront.

A first implication is a reframing of moral failure. Many familiar shortcomings in moral judgment—oversimplification, rigidity, or neglect of distant stakeholders—need not be attributed to defective values or irrationality. Instead, they can arise from rational allocations of limited resources across representation and inference. The breadth--depth tradeoff implies that attending to more entities necessarily constrains how deeply their interactions can be reasoned about. This does not excuse moral failings, but it shifts explanatory weight toward underlying computational constraints rather than attributing outcomes solely to character or principle.

Second, the framework offers a non-relativist account of persistent moral disagreement. Agents who share the same underlying moral objective may nonetheless reach different conclusions because they operate at different feasible points in the breadth--depth space or face different resource constraints. Disagreement, on this view, can reflect differences in approximation rather than differences in moral values. Persistent moral disagreement therefore need not imply moral pluralism, but may instead reflect stable variation in how agents approximate a common moral objective under constraint.

Third, the results offer a reinterpretation of ethical theories. Rather than treating utilitarianism, contractualism, or care ethics as competing foundational accounts, the framework treats them as families of reasoning strategies that are locally efficient under different constraints. Broad but shallow aggregation, narrow but deep reasoning, and heuristic rule-following each occupy different regions of the feasible space. This perspective does not adjudicate between theories, but it helps explain why different approaches appear compelling in different contexts.

The framework also introduces a new notion of moral progress. Progress need not consist in expanding the scope of moral concern, discovering new values, or converging on a single moral doctrine. Instead, progress can occur through improved abstractions, more efficient inference procedures, better data, or better allocation of moral reasoning effort across contexts—each of which reduces regret under fixed resources. This reframes historical and institutional moral change as, in part, advances in how moral reasoning is computed.

Finally, the implications for artificial systems are primarily diagnostic rather than prescriptive. Moral competence should not be reduced to a single objective score, but instead evaluated relative to resource budgets and deployment contexts along dimensions of scope, depth, and latency. Many alignment failures may reflect mismatches between moral computation and context rather than incorrect objectives. The framework provides a technical vocabulary for analyzing where and why morally competent behavior breaks down under constraint.

Overall, Bounded Morality does not aim to compete with moral philosophy or moral psychology. Its contribution is narrower: to isolate a structural feature that any account of moral reasoning must contend with when instantiated in finite agents, and to show how this feature shapes disagreement, heuristics, and the limits of moral deliberation.

\section{Conclusion}

This paper introduced \emph{Bounded Morality}, a framework that treats moral reasoning as constrained inference over structured moral worlds. By explicitly modeling representational scope, inferential depth, and computational cost, we showed that finite agents face an unavoidable tradeoff between moral breadth and moral depth. This tradeoff constrains which moral strategies are achievable under fixed resources, making explicit the limits of moral reasoning for finite agents. Conceptually, the framework reframes moral disagreement, heuristics, and ethical theories as predictable consequences of bounded computation rather than as failures of moral concern or coherence. For artificial systems, it highlights that alignment is not only a matter of specifying correct values, but also of designing moral computation that is well matched to capacity and context. More broadly, Bounded Morality provides a minimal computational lens through which insights from ethics, psychology, and AI can be integrated and evaluated under realistic constraints.

\begin{acknowledgments}
    This work was supported by the Amaranth Foundation. The authors thank Jared Moore and David Gottlieb for valuable discussions, including insights drawn from their instruction in Stanford’s \href{https://web.stanford.edu/class/cs186/}{CS186: How to Make a Moral Agent} course, and Nicholas Christakis for helpful conversations.
\end{acknowledgments}

\section*{Declaration on Generative AI}
 During the preparation of this work, the authors used GPT-5 in order to: Improve writing style, Content enhancement. After using this tool, the authors reviewed and edited the content as needed and take full responsibility for the publication’s content.

\bibliography{references}

@incollection{gottlieb2015aristotle,
  title={Aristotle: nicomachean ethics},
  author={Gottlieb, Paula},
  booktitle={Central Works of Philosophy v1},
  pages={46--68},
  year={2015},
  publisher={Routledge}
}

@book{scanlon2000we,
  title={What we owe to each other},
  author={Scanlon, Thomas M},
  year={2000},
  publisher={Belknap Press}
}

@incollection{kant2020groundwork,
  title={Groundwork of the Metaphysic of Morals},
  author={Kant, Immanuel},
  booktitle={Immanuel Kant},
  pages={17--98},
  year={2020},
  publisher={Routledge}
}

@incollection{mill2016utilitarianism,
  title={Utilitarianism},
  author={Mill, John Stuart},
  booktitle={Seven masterpieces of philosophy},
  pages={329--375},
  year={2016},
  publisher={Routledge}
}

@book{gigerenzer2011heuristics,
  title={Heuristics: The foundations of adaptive behavior.},
  author={Gigerenzer, Gerd Ed and Hertwig, Ralph Ed and Pachur, Thorsten Ed},
  year={2011},
  publisher={Oxford university press}
}

@book{selman_growth_1980,
    title = {The {Growth} of {Interpersonal} {Understanding}: {Developmental} and {Clinical} {Analyses}},
    isbn = {978-0-12-636450-7},
    shorttitle = {The {Growth} of {Interpersonal} {Understanding}},
    language = {en},
    publisher = {Academic Press},
    author = {Selman, Robert L.},
    year = {1980},
}

@article{haidt_emotional_2001,
    title = {The emotional dog and its rational tail: {A} social intuitionist approach to moral judgment},
    volume = {108},
    issn = {1939-1471},
    shorttitle = {The emotional dog and its rational tail},
    doi = {10.1037/0033-295X.108.4.814},
    number = {4},
    journal = {Psychological Review},
    author = {Haidt, Jonathan},
    year = {2001},
    note = {Place: US
Publisher: American Psychological Association},
    pages = {814--834},
}

@book{arrow_social_1970,
    series = {Cowles {Foundation} {Monograph} {Series}},
    title = {Social {Choice} and {Individual} {Values}},
    isbn = {978-0-300-01363-4},
    url = {https://books.google.com/books?id=uebtAAAAMAAJ},
    publisher = {Yale University Press},
    author = {Arrow, K.J.},
    year = {1970},
}

@book{brandt_handbook_2016,
    address = {Cambridge ; New York},
    title = {Handbook of computational social choice},
    isbn = {978-1-107-06043-2},
    publisher = {Cambridge University Press},
    editor = {Brandt, Felix},
    year = {2016},
}

@book{singer_expanding_1981,
    title = {The expanding circle},
    isbn = {0-19-824646-3},
    publisher = {Clarendon Press Oxford},
    author = {Singer, Peter},
    year = {1981},
}

@article{kohlberg_moral_1976,
    title = {Moral stages and moralization: {The} cognitive-development approach},
    journal = {Moral development and behavior: Theory research and social issues},
    author = {Kohlberg, Lawrence},
    year = {1976},
    note = {Publisher: Rinehart \& Winston},
    pages = {31--53},
}

@book{piaget_moral_2013,
    title = {The moral judgment of the child},
    isbn = {1-315-00968-4},
    publisher = {Routledge},
    author = {Piaget, Jean},
    year = {2013},
}

@incollection{eisenberg_prosocial_2006,
    address = {Hoboken, NJ, US},
    title = {Prosocial {Development}},
    isbn = {978-0-471-27290-8},
    booktitle = {Handbook of child psychology: {Social}, emotional, and personality development, {Vol}. 3, 6th ed},
    publisher = {John Wiley \& Sons, Inc.},
    author = {Eisenberg, Nancy and Fabes, Richard A. and Spinrad, Tracy L.},
    year = {2006},
    pages = {646--718},
}

@book{gilligan_different_1993,
    title = {In a different voice: {Psychological} theory and women’s development},
    isbn = {0-674-44544-9},
    publisher = {Harvard university press},
    author = {Gilligan, Carol},
    year = {1993},
}

@book{gardiner_perfect_2011,
    title = {A perfect moral storm: {The} ethical tragedy of climate change},
    isbn = {0-19-991045-6},
    publisher = {Oxford University Press},
    author = {Gardiner, Stephen M},
    year = {2011},
}

@book{parfit_reasons_1986,
    edition = {1},
    title = {Reasons and {Persons}},
    isbn = {978-0-19-824908-5 978-0-19-159817-3},
    url = {https://academic.oup.com/book/12484},
    language = {en},
    urldate = {2025-11-05},
    publisher = {Oxford University PressOxford},
    author = {Parfit, Derek},
    month = jan,
    year = {1986},
    doi = {10.1093/019824908X.001.0001},
}

@article{crimston_moral_2016,
    title = {Moral expansiveness: {Examining} variability in the extension of the moral world.},
    volume = {111},
    issn = {1939-1315},
    number = {4},
    journal = {Journal of personality and social psychology},
    author = {Crimston, Charlie R and Bain, Paul G and Hornsey, Matthew J and Bastian, Brock},
    year = {2016},
    note = {Publisher: American Psychological Association},
    pages = {636},
}

@article{greene_neural_2004,
    title = {The {Neural} {Bases} of {Cognitive} {Conflict} and {Control} in {Moral} {Judgment}},
    volume = {44},
    issn = {08966273},
    url = {https://linkinghub.elsevier.com/retrieve/pii/S0896627304006348},
    doi = {10.1016/j.neuron.2004.09.027},
    language = {en},
    number = {2},
    urldate = {2025-11-05},
    journal = {Neuron},
    author = {Greene, Joshua D. and Nystrom, Leigh E. and Engell, Andrew D. and Darley, John M. and Cohen, Jonathan D.},
    month = oct,
    year = {2004},
    pages = {389--400},
}

@misc{takeshita_towards_2023,
    title = {Towards {Theory}-based {Moral} {AI}: {Moral} {AI} with {Aggregating} {Models} {Based} on {Normative} {Ethical} {Theory}},
    shorttitle = {Towards {Theory}-based {Moral} {AI}},
    url = {http://arxiv.org/abs/2306.11432},
    doi = {10.48550/arXiv.2306.11432},
    urldate = {2025-11-06},
    publisher = {arXiv},
    author = {Takeshita, Masashi and Rafal, Rzepka and Araki, Kenji},
    month = jun,
    year = {2023},
    note = {arXiv:2306.11432 [cs]},
}

@inproceedings{hegde_ethics_2020,
    address = {Yokohama, Japan},
    title = {Ethics, {Prosperity}, and {Society}: {Moral} {Evaluation} {Using} {Virtue} {Ethics} and {Utilitarianism}},
    isbn = {978-0-9992411-6-5},
    shorttitle = {Ethics, {Prosperity}, and {Society}},
    url = {https://www.ijcai.org/proceedings/2020/24},
    doi = {10.24963/ijcai.2020/24},
    language = {en},
    urldate = {2025-11-06},
    booktitle = {Proceedings of the {Twenty}-{Ninth} {International} {Joint} {Conference} on {Artificial} {Intelligence}},
    publisher = {International Joint Conferences on Artificial Intelligence Organization},
    author = {Hegde, Aditya and Agarwal, Vibhav and Rao, Shrisha},
    month = jul,
    year = {2020},
    pages = {167--174},
}

@article{white_mapping_2024,
    title = {Mapping the ethic-theoretical foundations of artificial intelligence research},
    volume = {66},
    copyright = {© 2024 The Authors. Thunderbird International Business Review published by Wiley Periodicals LLC.},
    issn = {1520-6874},
    url = {https://onlinelibrary.wiley.com/doi/abs/10.1002/tie.22368},
    doi = {10.1002/tie.22368},
    language = {en},
    number = {2},
    urldate = {2025-11-06},
    journal = {Thunderbird International Business Review},
    author = {White, Gareth R. T. and Samuel, Anthony and Jones, Paul and Madhavan, Naveen and Afolayan, Ademola and Abdullah, Ahmed and Kaushik, Tanmay},
    year = {2024},
    note = {\_eprint: https://onlinelibrary.wiley.com/doi/pdf/10.1002/tie.22368},
    pages = {171--183},
}

@inproceedings{preniqi_moralbert_2024,
    address = {Bremen Germany},
    title = {{MoralBERT}: {A} {Fine}-{Tuned} {Language} {Model} for {Capturing} {Moral} {Values} in {Social} {Discussions}},
    isbn = {979-8-4007-1094-0},
    shorttitle = {{MoralBERT}},
    url = {https://dl.acm.org/doi/10.1145/3677525.3678694},
    doi = {10.1145/3677525.3678694},
    language = {en},
    urldate = {2025-11-05},
    booktitle = {Proceedings of the 2024 {International} {Conference} on {Information} {Technology} for {Social} {Good}},
    publisher = {ACM},
    author = {Preniqi, Vjosa and Ghinassi, Iacopo and Ive, Julia and Saitis, Charalampos and Kalimeri, Kyriaki},
    month = sep,
    year = {2024},
    pages = {433--442},
}

@incollection{simon_bounded_1990,
    address = {London},
    title = {Bounded {Rationality}},
    isbn = {978-1-349-20568-4},
    url = {https://doi.org/10.1007/978-1-349-20568-4_5},
    language = {en},
    urldate = {2025-11-06},
    booktitle = {Utility and {Probability}},
    publisher = {Palgrave Macmillan UK},
    author = {Simon, Herbert A.},
    editor = {Eatwell, John and Milgate, Murray and Newman, Peter},
    year = {1990},
    doi = {10.1007/978-1-349-20568-4_5},
    pages = {15--18},
}

@article{gabriel_artificial_2020,
    title = {Artificial intelligence, values, and alignment},
    volume = {30},
    issn = {0924-6495},
    number = {3},
    journal = {Minds and machines},
    publisher = {Springer},
    author = {Gabriel, Iason},
    year = {2020},
    pages = {411--437},
}

@book{piaget_origins_1952,
    title = {The origins of intelligence in children},
    volume = {8},
    publisher = {International universities press New York},
    author = {Piaget, Jean and Cook, Margaret},
    year = {1952},
    note = {Issue: 5},
}

@article{waytz_who_2010,
    title = {Who {Sees} {Human}?: {The} {Stability} and {Importance} of {Individual} {Differences} in {Anthropomorphism}},
    volume = {5},
    copyright = {https://journals.sagepub.com/page/policies/text-and-data-mining-license},
    issn = {1745-6916, 1745-6924},
    shorttitle = {Who {Sees} {Human}?},
    url = {https://journals.sagepub.com/doi/10.1177/1745691610369336},
    doi = {10.1177/1745691610369336},
    language = {en},
    number = {3},
    urldate = {2026-01-21},
    journal = {Perspectives on Psychological Science},
    author = {Waytz, Adam and Cacioppo, John and Epley, Nicholas},
    month = may,
    year = {2010},
    pages = {219--232},
}

@article{greene_cognitive_2008,
    title = {Cognitive load selectively interferes with utilitarian moral judgment},
    volume = {107},
    issn = {0010-0277},
    doi = {10.1016/j.cognition.2007.11.004},
    language = {eng},
    number = {3},
    journal = {Cognition},
    author = {Greene, Joshua D. and Morelli, Sylvia A. and Lowenberg, Kelly and Nystrom, Leigh E. and Cohen, Jonathan D.},
    month = jun,
    year = {2008},
    pages = {1144--1154},
}

@article{conway_deontological_2013,
    title = {Deontological and utilitarian inclinations in moral decision making: a process dissociation approach},
    volume = {104},
    issn = {1939-1315},
    shorttitle = {Deontological and utilitarian inclinations in moral decision making},
    doi = {10.1037/a0031021},
    language = {eng},
    number = {2},
    journal = {Journal of Personality and Social Psychology},
    author = {Conway, Paul and Gawronski, Bertram},
    month = feb,
    year = {2013},
    pages = {216--235},
}

@article{greene_beyond_2014,
    title = {Beyond {Point}-and-{Shoot} {Morality}: {Why} {Cognitive} ({Neuro}){Science} {Matters} for {Ethics}},
    volume = {124},
    issn = {0014-1704, 1539-297X},
    shorttitle = {Beyond {Point}-and-{Shoot} {Morality}},
    url = {https://www.journals.uchicago.edu/doi/10.1086/675875},
    doi = {10.1086/675875},
    language = {en},
    number = {4},
    urldate = {2026-01-21},
    journal = {Ethics},
    author = {Greene, Joshua D.},
    month = jul,
    year = {2014},
    pages = {695--726},
}

@article{kahane_beyond_2018,
    address = {US},
    title = {Beyond sacrificial harm: {A} two-dimensional model of utilitarian psychology},
    volume = {125},
    issn = {1939-1471},
    shorttitle = {Beyond sacrificial harm},
    doi = {10.1037/rev0000093},
    number = {2},
    journal = {Psychological Review},
    publisher = {American Psychological Association},
    author = {Kahane, Guy and Everett, Jim A. C. and Earp, Brian D. and Caviola, Lucius and Faber, Nadira S. and Crockett, Molly J. and Savulescu, Julian},
    year = {2018},
    pages = {131--164},
}

@article{paxton_reflection_2012,
    address = {United Kingdom},
    title = {Reflection and reasoning in moral judgment},
    volume = {36},
    issn = {1551-6709},
    doi = {10.1111/j.1551-6709.2011.01210.x},
    number = {1},
    journal = {Cognitive Science},
    publisher = {Wiley-Blackwell Publishing Ltd.},
    author = {Paxton, Joseph M. and Ungar, Leo and Greene, Joshua D.},
    year = {2012},
    pages = {163--177},
}

@article{suter_time_2011,
    address = {Netherlands},
    title = {Time and moral judgment},
    volume = {119},
    issn = {1873-7838},
    doi = {10.1016/j.cognition.2011.01.018},
    number = {3},
    journal = {Cognition},
    publisher = {Elsevier Science},
    author = {Suter, Renata S. and Hertwig, Ralph},
    year = {2011},
    pages = {454--458},
}

@article{baron_protected_1997,
    title = {Protected {Values}},
    volume = {70},
    issn = {1096-0341},
    doi = {10.1006/obhd.1997.2690},
    language = {eng},
    number = {1},
    journal = {Virology},
    author = {Baron, J. and Spranca, M.},
    month = apr,
    year = {1997},
    pages = {1--16},
}

@article{fetherstonhaugh_insensitivity_1997,
    title = {Insensitivity to the {Value} of {Human} {Life}: {A} {Study} of {Psychophysical} {Numbing}},
    volume = {14},
    issn = {1573-0476},
    shorttitle = {Insensitivity to the {Value} of {Human} {Life}},
    url = {https://doi.org/10.1023/A:1007744326393},
    doi = {10.1023/A:1007744326393},
    language = {en},
    number = {3},
    urldate = {2026-01-21},
    journal = {Journal of Risk and Uncertainty},
    author = {FETHERSTONHAUGH, DAVID and SLOVIC, PAUL and JOHNSON, STEPHEN and FRIEDRICH, JAMES},
    month = may,
    year = {1997},
    pages = {283--300},
}

@article{dickert_scope_2015,
    series = {Modeling and {Aiding} {Intuition} in {Organizational} {Decision} {Making}},
    title = {Scope insensitivity: {The} limits of intuitive valuation of human lives in public policy},
    volume = {4},
    issn = {2211-3681},
    shorttitle = {Scope insensitivity},
    url = {https://www.sciencedirect.com/science/article/pii/S2211368114000795},
    doi = {10.1016/j.jarmac.2014.09.002},
    number = {3},
    urldate = {2026-01-21},
    journal = {Journal of Applied Research in Memory and Cognition},
    author = {Dickert, Stephan and Västfjäll, Daniel and Kleber, Janet and Slovic, Paul},
    month = sep,
    year = {2015},
    pages = {248--255},
}

@article{buon_non-mentalistic_2013,
    title = {A non-mentalistic cause-based heuristic in human social evaluations},
    volume = {126},
    issn = {1873-7838},
    doi = {10.1016/j.cognition.2012.09.006},
    language = {eng},
    number = {2},
    journal = {Cognition},
    author = {Buon, Marine and Jacob, Pierre and Loissel, Elsa and Dupoux, Emmanuel},
    month = feb,
    year = {2013},
    pages = {149--155},
}

@article{martin_effect_2021,
    title = {The {Effect} of {Cognitive} {Load} on {Intent}-{Based} {Moral} {Judgment}},
    volume = {45},
    copyright = {© 2021 Cognitive Science Society LLC},
    issn = {1551-6709},
    url = {https://onlinelibrary.wiley.com/doi/abs/10.1111/cogs.12965},
    doi = {10.1111/cogs.12965},
    language = {en},
    number = {4},
    urldate = {2026-01-21},
    journal = {Cognitive Science},
    author = {Martin, Justin W. and Buon, Marine and Cushman, Fiery},
    year = {2021},
    note = {\_eprint: https://onlinelibrary.wiley.com/doi/pdf/10.1111/cogs.12965},
    pages = {e12965},
}

@inproceedings{conitzer_communication_2005,
    address = {Vancouver BC Canada},
    title = {Communication complexity of common voting rules},
    isbn = {978-1-59593-049-1},
    url = {https://dl.acm.org/doi/10.1145/1064009.1064018},
    doi = {10.1145/1064009.1064018},
    language = {en},
    urldate = {2026-01-21},
    booktitle = {Proceedings of the 6th {ACM} conference on {Electronic} commerce},
    publisher = {ACM},
    author = {Conitzer, Vincent and Sandholm, Tuomas},
    month = jun,
    year = {2005},
    pages = {78--87},
}

@misc{conitzer_social_2024,
    title = {Social {Choice} {Should} {Guide} {AI} {Alignment} in {Dealing} with {Diverse} {Human} {Feedback}},
    url = {http://arxiv.org/abs/2404.10271},
    doi = {10.48550/arXiv.2404.10271},
    urldate = {2026-01-17},
    publisher = {arXiv},
    author = {Conitzer, Vincent and Freedman, Rachel and Heitzig, Jobst and Holliday, Wesley H. and Jacobs, Bob M. and Lambert, Nathan and Mossé, Milan and Pacuit, Eric and Russell, Stuart and Schoelkopf, Hailey and Tewolde, Emanuel and Zwicker, William S.},
    month = jun,
    year = {2024},
    note = {arXiv:2404.10271 [cs]},
}

@article{levine_resource_2025,
    title = {Resource {Rational} {Contractualism} {Should} {Guide} {AI} {Alignment}},
    journal = {arXiv preprint arXiv:2506.17434},
    author = {Levine, Sydney and Franklin, Matija and Zhi-Xuan, Tan and Guyot, Secil Yanik and Wong, Lionel and Kilov, Daniel and Choi, Yejin and Tenenbaum, Joshua B and Goodman, Noah and Lazar, Seth},
    year = {2025},
}

@article{levine_resource-rational_2024,
    title = {Resource-rational contractualism: {A} triple theory of moral cognition},
    issn = {0140-525X, 1469-1825},
    shorttitle = {Resource-rational contractualism},
    url = {https://www.cambridge.org/core/journals/behavioral-and-brain-sciences/article/abs/resourcerational-contractualism-a-triple-theory-of-moral-cognition/5A567D41A472DBC0D965460966580C74},
    doi = {10.1017/S0140525X24001067},
    language = {en},
    urldate = {2026-01-21},
    journal = {Behavioral and Brain Sciences},
    author = {Levine, Sydney and Chater, Nick and Tenenbaum, Joshua B. and Cushman, Fiery},
    month = oct,
    year = {2024},
    pages = {1--38},
}

@article{griffiths_rational_2015,
    title = {Rational {Use} of {Cognitive} {Resources}: {Levels} of {Analysis} {Between} the {Computational} and the {Algorithmic}},
    volume = {7},
    issn = {1756-8757, 1756-8765},
    shorttitle = {Rational {Use} of {Cognitive} {Resources}},
    url = {https://onlinelibrary.wiley.com/doi/10.1111/tops.12142},
    doi = {10.1111/tops.12142},
    language = {en},
    number = {2},
    urldate = {2026-01-21},
    journal = {Topics in Cognitive Science},
    author = {Griffiths, Thomas L. and Lieder, Falk and Goodman, Noah D.},
    month = apr,
    year = {2015},
    pages = {217--229},
}

@book{anderson_adaptive_2013,
    edition = {1},
    title = {The {Adaptive} {Character} of {Thought}},
    isbn = {978-0-203-77173-0},
    url = {https://www.taylorfrancis.com/books/9780203771730},
    doi = {10.4324/9780203771730},
    language = {en},
    urldate = {2026-01-21},
    publisher = {Psychology Press},
    author = {Anderson, John R.},
    month = jan,
    year = {2013},
}

@incollection{chugh_bounded_2005,
    address = {New York, NY, US},
    title = {Bounded {Ethicality} as a {Psychological} {Barrier} to {Recognizing} {Conflicts} of {Interest}},
    isbn = {978-0-521-84439-0},
    doi = {10.1017/CBO9780511610332.006},
    booktitle = {Conflicts of interest: {Challenges} and solutions in business, law, medicine, and public policy},
    publisher = {Cambridge University Press},
    author = {Chugh, Dolly and Bazerman, Max H. and Banaji, Mahzarin R.},
    year = {2005},
    pages = {74--95},
}

@article{tenbrunsel_13_2008,
    title = {13 {Ethical} {Decision} {Making}: {Where} {We}’ve {Been} and {Where} {We}’re {Going}},
    volume = {2},
    issn = {1941-6520, 1941-6067},
    shorttitle = {13 {Ethical} {Decision} {Making}},
    url = {http://journals.aom.org/doi/10.5465/19416520802211677},
    doi = {10.5465/19416520802211677},
    language = {en},
    number = {1},
    urldate = {2026-01-21},
    journal = {Academy of Management Annals},
    author = {Tenbrunsel, Ann E. and Smith‐Crowe, Kristin},
    month = jan,
    year = {2008},
    pages = {545--607},
}

@book{marr_vision_2010,
    address = {Cambridge, Mass},
    title = {Vision: a computational investigation into the human representation and processing of visual information},
    isbn = {978-0-262-51462-0 978-0-262-28961-0},
    shorttitle = {Vision},
    language = {eng},
    publisher = {MIT Press},
    author = {Marr, David},
    year = {2010},
}

@article{sunstein_social_1996,
    title = {Social {Norms} and {Social} {Roles}},
    volume = {96},
    issn = {00101958},
    url = {https://www.jstor.org/stable/1123430?origin=crossref},
    doi = {10.2307/1123430},
    number = {4},
    urldate = {2026-01-21},
    journal = {Columbia Law Review},
    author = {Sunstein, Cass R.},
    month = may,
    year = {1996},
    pages = {903},
}

\newpage
\appendix

\section{Formal Definitions and Notation}

\begin{table}[h!]
\centering
\small
\begin{tabular}{@{}llp{6.2cm}@{}}
\toprule
Symbol / Term & Name & Definition / Role \\
\midrule
$G^\star=(V^\star,E^\star)$ 
& Moral interaction graph 
& Full graph of morally relevant entities and their local influence relations. \\

$v\in V^\star$ 
& Moral entity 
& Node whose state contributes directly to moral welfare. \\

$\mathcal{S}_v$ 
& Local state space 
& State space associated with entity $v$. \\

$A$ 
& Action set 
& Finite set of atomic interventions available at each time step. \\

$\mathcal{A}_H=A^H$ 
& Action sequences 
& Length-$H$ intervention sequences (temporally extended control policies). \\

$F$ 
& Moral dynamics 
& Local graph-structured state update rule governing trajectory evolution. \\

$U(s)$ 
& Instantaneous welfare 
& Graph-structured welfare function decomposing into node and edge potentials. \\

$M^\star(\alpha,\omega)$ 
& Ground-truth moral value 
& Infinite-horizon discounted welfare induced by action sequence $\alpha$. \\

\addlinespace
$\pi_V:V^\star\to V$ 
& Aggregation map 
& Surjection defining coarse-grained moral representation. \\

$G=(V,E)$ 
& Abstract moral representation 
& Coarse interaction graph induced by aggregation. \\

$b(G)=|V|+|E|$ 
& Breadth 
& Representational complexity of abstraction $G$. \\

$H$ 
& Depth 
& Rollout horizon for consequence propagation. \\

$\hat{M}_{G,H}$ 
& Truncated objective 
& Finite-horizon approximation to $M^\star$ under abstraction $G$. \\

\addlinespace
$\mathrm{Cost}_{\mathrm{info}}(G)$ 
& Informational cost 
& Cost of encoding representation $G$; increasing in $b(G)$. \\

$\mathrm{Cost}_{\mathrm{infer}}(G,H)$ 
& Inferential cost 
& Cost of simulating and optimizing over length-$H$ trajectories under $G$. \\

$\mathrm{Cost}(G,H)$ 
& Total cost 
& Sum of informational and inferential costs. \\

$B$ 
& Resource budget 
& Upper bound on allowable computational cost. \\

$\rho$ 
& Allocation rule 
& Strategy mapping $(\omega,B)$ to feasible $(G,H)$. \\

$\pi^\rho_B$ 
& Bounded moral policy 
& Intervention sequence selected under allocation rule $\rho$. \\

$R(\omega;\rho,B)$ 
& State-dependent regret 
& Loss in ground-truth value relative to optimal unbounded intervention. \\

$\mathsf{P}$ 
& World distribution 
& Distribution over world states used to evaluate expected regret. \\

\bottomrule
\end{tabular}
\caption{\textbf{Comprehensive glossary of formal objects used in the Bounded Morality framework.} This table consolidates the mathematical definitions introduced in Section~\ref{sec:bounded} for reference.}
\label{tab:bounded_morality_glossary}
\end{table}

\newpage

\section{A Worked Example of Bounded Morality}
\label{app:worked_example}

We illustrate the Bounded Morality framework with a minimal model of content moderation that instantiates the moral interaction graph, intervention-driven dynamics, graph-structured welfare, abstraction, and regret defined in Section~\ref{sec:bounded}. The purpose is to show explicitly how a concrete policy problem maps onto the formal objects of the theory. The example demonstrates that computational constraints—specifically limited inferential depth and limited representational breadth—can invert moral decisions even when the welfare function and dynamics are fully known and optimization within a fixed representation is exact.

Two phenomena emerge:

\begin{enumerate}
\item \textbf{Depth Reversal:} Truncated evaluation ($H$ small) favors aggressive intervention, while longer rollouts ($H$ large) favor targeted intervention.
\item \textbf{Breadth Constraint:} Coarse abstraction can eliminate the targeted intervention from the admissible action set.
\end{enumerate}

\subsection{Instantiation of the Moral System}

\paragraph{Moral Interaction Graph.}
Let
\[
V^\star=\{E_1,E_2,C,M_1,M_2\},
\]
representing two extremist nodes, a connector node, and two moderate nodes. 

The moral interaction graph $G^\star=(V^\star,E^\star)$ is
\[
E^\star=
\big\{
\{E_1,E_2\},
\{E_1,C\},\{E_2,C\},
\{C,M_1\},\{C,M_2\},
\{M_1,M_2\}
\big\}.
\]

The connector $C$ is the unique pathway through which polarization flows from extremists to moderates. 

\paragraph{State Space.}
Each node $v\in V^\star$ has polarization state $s_v(t)\in\mathbb{R}$ and resentment state $r_v(t)\in\mathbb{R}$. 
The full moral state is 
$s(t)=(s_v(t))_{v\in V^\star}\in\mathbb{R}^5$.

\paragraph{Actions.}
An intervention is a sanction vector
$a\in A=\{0,1\}^5,$
applied at $t=0$ only. 
Thus, although the general framework allows $\alpha\in A^H$, here we reduce $\mathcal{A}_H(G^\star)$ to a single sanction decision.

Sanctioning has two effects:
\begin{align}
s_v(0) &=
\begin{cases}
s_{\mathrm{san}} & a_v=1,\\
\omega_v & a_v=0,
\end{cases}
\\
r(0) &= r_0 a.
\end{align}

Sanctioning reduces polarization immediately but induces resentment.

\paragraph{Intervention-Driven Dynamics.}

Let $W$ be the adjacency matrix of $G^\star$ with unit edge weights and self-loops. 
Define a sanction-modulated influence matrix:
\[
W(a)=W\,\mathrm{diag}(w(a)),
\quad
w_u(a)=
\begin{cases}
k & a_u=1,\\
1 & a_u=0,
\end{cases}
\]
and normalize rows to obtain $P(a)$.

Polarization and resentment evolve as:
\begin{align}
s(t+1) &= \beta P(a)s(t)+\alpha r(t), \\
r(t+1) &= (1-\delta)r(t).
\end{align}

Because $P(a)$ is row-stochastic and $\beta<1$, the spectral radius of $\beta P(a)$ is strictly less than one, i.e., the system is globally stable. 
Reversal effects therefore arise from truncation rather than instability.

\paragraph{Graph-Structured Welfare.}

Let $n=|V^\star|$ and $m= |E^\star|$. Define instantaneous welfare:
\[
U(s)
=
-\frac1n\sum_{v\in V^\star} s_v^2
-
\frac{\lambda}{m}
\sum_{\{u,v\}\in E^\star}(s_u-s_v)^2,
\]

The first term penalizes polarization magnitude; the second penalizes disagreement across edges.

Ground-truth moral value is
\[
M^\star(a,\omega)
=
\sum_{t=0}^{\infty}\gamma^t U(s(t)).
\]

\paragraph{Finite-Horizon Approximation.}

A bounded planner with depth $H$ evaluates:
\[
\hat M_{G^\star,H}(a,\omega)
=
\sum_{t=0}^{H}\gamma^t U(s(t)).
\]

Truncation introduces approximation error relative to $M^\star$.

\subsection{Numerical Illustration}

Parameters:
\begin{align*}
\beta &= 0.9, & \alpha &= 0.2, \\
\delta &= 0.05, & r_0 &= 1.0, \\
k &= 0.1, & \gamma &= 0.9, \\
\lambda &= 0.1, & s_{\mathrm{san}} &= 0.2.
\end{align*}

Initial polarization:
\[
\omega_{E_1}=\omega_{E_2}=1,\quad
\omega_C=0.6,\quad
\omega_{M_1}=\omega_{M_2}=0.
\]

Candidate actions:
\[
a^{(0)}=\mathbf{0},\quad
a^{(E)}=\mathbf{1}_{\{E_1,E_2\}},\quad
a^{(EC)}=\mathbf{1}_{\{E_1,E_2,C\}}.
\]

\begin{table}[h!]
\centering
\begin{tabular}{c|ccc}
$H$ & $a^{(0)}$ & $a^{(E)}$ & $a^{(EC)}$ \\
\hline
2  & -0.765 & -0.297 & \textbf{-0.088} \\
10 & -1.304 & \textbf{-0.581} & -0.715 \\
\infty & -1.348 & \textbf{-0.634} & -0.955 \\
\end{tabular}
\caption{Finite-horizon value $\hat M_{G^\star,H}(a,\omega)$ and infinite-horizon value $M^\star(a,\omega)$ (approximated using rollout to 2000 steps). Bold indicates the maximizer.}
\end{table}

\begin{itemize}
    \item At depth $H=2$, sanctioning extremists and the connector is optimal: $\pi_{G^\star,2}(\omega)=a^{(EC)}.$
    \item At depth $H=10$, sanctioning extremists only is optimal: $\pi_{G^\star,10}(\omega)=a^{(E)}.$
    \item Under the infinite-horizon objective $M^\star$, sanctioning extremists only is optimal: $a^\star(\omega)=a^{(E)}.$
\end{itemize}

The depth-$2$ planner instead selects $a^{(EC)}$, yielding state-dependent regret
\[
R(\omega;\rho^{(2)},B)
=
M^\star(a^{(E)},\omega)
-
M^\star(a^{(EC)},\omega)
\approx
0.322.
\]
By contrast, the depth-$10$ planner selects $a^{(E)}$ and therefore incurs zero regret within this action class.

\paragraph{Mechanism of Depth Reversal.}
With a short horizon, the planner mainly sees the immediate benefit of cutting the bridge between extremists and moderates: sanctioning $C$ quickly reduces visible spillover, which appears strongly beneficial in the first few steps. However, sanctioning $C$ also creates resentment at a highly connected node. That resentment spreads gradually through all of $C$’s links before fading, eventually affecting the entire network. A shallow evaluation captures the fast reduction in spillover but misses this slower, system-wide backlash. A deeper evaluation sees both effects, and once the delayed costs are included, targeted sanctioning becomes preferable.

\begin{figure}[t]
\centering
\begin{tikzpicture}[scale=1.1, every node/.style={font=\small}]

\node[draw,circle,minimum size=7mm] (E1) at (0,2) {$E_1$};
\node[draw,circle,minimum size=7mm] (E2) at (2,2) {$E_2$};
\node[draw,circle,minimum size=7mm] (C)  at (1,1) {$C$};
\node[draw,circle,minimum size=7mm] (M1) at (0,0) {$M_1$};
\node[draw,circle,minimum size=7mm] (M2) at (2,0) {$M_2$};

\draw (E1)--(E2);
\draw (E1)--(C);
\draw (E2)--(C);
\draw (C)--(M1);
\draw (C)--(M2);
\draw (M1)--(M2);

\node at (1,-0.8) {$G^\star$};

\draw[->,thick] (3.2,1) -- (4.4,1);
\node at (3.8,1.3) {$\pi_V$};

\node[draw,circle,minimum size=9mm] (X) at (6,1.6) {$X$};
\node[draw,circle,minimum size=9mm] (Y) at (6,0.2) {$Y$};

\draw (X)--(Y);

\node at (6,-0.8) {$G$};

\end{tikzpicture}
\caption{
Ground-truth moral interaction graph $G^\star$ (left) and a coarse abstraction $G$ (right) induced by aggregation map $\pi_V$ with $X=\{E_1,E_2,C\}$ and $Y=\{M_1,M_2\}$.
}
\label{fig:worked_example}
\end{figure}
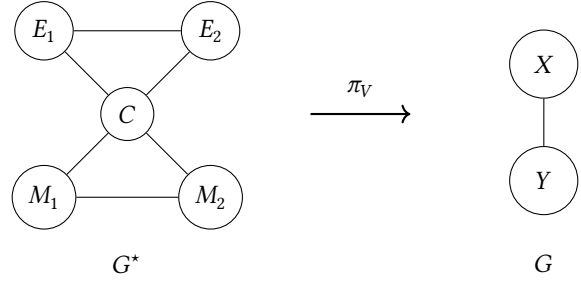

\subsection*{Breadth Constraints}

We now examine how abstraction alters the feasible intervention set and therefore the attainable moral value.

\paragraph{Medium abstraction (3 nodes).}
Let
\[
\pi_V(E_1)=\pi_V(E_2)=X_E,\quad
\pi_V(C)=X_C,\quad
\pi_V(M_1)=\pi_V(M_2)=X_M.
\]
This representation preserves the morally relevant distinction between extremists and the connector. 
Under this abstraction, the action $a^{(E)}$ remains admissible, and with sufficient depth the planner can recover the infinite-horizon optimum.

\paragraph{Coarse abstraction (2 nodes).}
Let
\[
\pi_V(E_1)=\pi_V(E_2)=\pi_V(C)=X,\quad
\pi_V(M_1)=\pi_V(M_2)=Y.
\]
Admissible actions must satisfy
\[
a_{E_1}=a_{E_2}=a_C,
\]
so $a^{(E)}\notin\mathcal{A}_H(G_{\text{coarse}})$.

Under this abstraction (Figure \ref{fig:worked_example}), the planner cannot implement the infinite-horizon optimal action. 
Even with large depth, the best admissible intervention coincides with $a^{(EC)}$, whose infinite-horizon value is
\[
M^\star(a^{(EC)},\omega)\approx -0.955,
\]
strictly below the optimum
\[
M^\star(a^{(E)},\omega)\approx -0.634.
\]
The abstraction therefore induces an irreducible regret of approximately $0.322$, independent of depth.

\subsection*{Regret and Moral Progress}

Let $\rho^{(2)}(\omega,B)=(G^\star,2)$ and
$\rho^{(10)}(\omega,B)=(G^\star,10)$, assuming the budget permits either allocation.

Using the infinite-horizon values above,
\[
R(\omega;\rho^{(2)},B)\approx 0.322,
\qquad
R(\omega;\rho^{(10)},B)=0.
\]

Under the coarse abstraction $G_{\text{coarse}}$, regret of the same magnitude persists even with large depth, since the optimal action is infeasible.

\bigskip
\noindent
\textbf{Takeaway.}
Limited depth hides delayed consequences.
Limited breadth hides feasible interventions.
Both forms of regret arise from constrained computation rather than disagreement about values.
Moral progress corresponds to improving how computational resources are allocated across these two dimensions.

\end{document}